%% file: main.tex
\documentclass{article}
\usepackage[final]{neurips_2023}

\usepackage[utf8]{inputenc} % allow utf-8 input
\usepackage[T1]{fontenc}    % use 8-bit T1 fonts
\usepackage{hyperref}       % hyperlinks
\usepackage{url}            % simple URL typesetting
\usepackage{booktabs}       % professional-quality tables
\usepackage{amsfonts}       % blackboard math symbols
\usepackage{nicefrac}       % compact symbols for 1/2, etc.
\usepackage{microtype}      % microtypography
\usepackage{xcolor}         % colors

\usepackage[shortlabels]{enumitem}
\usepackage{graphicx}
\usepackage{amsthm}
\input{macros}

\title{The Adversarial Consistency of Surrogate Risks for Binary Classification}
	
\author{%
  Natalie S.~Frank \\
  Courant Institute\\
  New York University\\
  New York, NY 10012 \\
  \texttt{nf1066@nyu.edu} \\
   \And
   Jonathan Niles-Weed \\
  Courant Institute\\
  New York University\\
  New York, NY 10012 \\
  \texttt{jnw@cims.nyu.edu}
}

\begin{document}

\maketitle

\input{Sections/1-Introduction.tex}
\input{Sections/2-Related_works.tex}
\input{Sections/3-background}

\input{Sections/4-minimizing_sequences.tex}

\input{Sections/5-Towards_H_consistency.tex}

\input{Sections/6-conclusion.tex}

\input{Sections/7-acknowledgements}
\bibliographystyle{abbrvnat}
\bibliography{bibliography,bib3}

\vfill
\pagebreak
\appendix
\input{Appendices/1-Alt_consistency.tex}
\input{Appendices/2-ov_R_to_R}
\input{Appendices/3-phi_properties.tex}
\input{Appendices/4-W_infty_couplings.tex}

\input{Appendices/5-minimax_classification_proof.tex}
\input{Appendices/6-ineq_to_eq}

\end{document}

%% file: macros.tex
\usepackage{amsfonts}
\usepackage{mathrsfs}
\usepackage{amsmath}
\usepackage{color}
\definecolor{blue}{rgb}{0.0, 0.0, 1.0}

\definecolor{jgGreen}{rgb}{0.0, 0.5, 0.0}
\definecolor{pink}{rgb}{0.5, 0.0, 0.25}

\newcommand{\ignore}[1]{}%for temporarily taking out text

\newcommand{\prm}{R_\phi^\e}
\newcommand{\dl}{\bar R_\phi}
\newcommand{\cprm}{R^\e}
\newcommand{\cdl}{\bar R}

\newcommand{\cB}{{\mathcal B}}

\newcommand{\cD}{\mathcal{D}}

\newcommand{\cH}{\mathcal{H}}

\newcommand{\bh}{{\mathbf h}}

\newcommand{\bx}{{\mathbf x}}
\newcommand{\by}{{\mathbf y}}

\DeclareMathOperator*{\argmin}{argmin}

\DeclareMathOperator*{\esssup}{ess\,sup}

\DeclareMathOperator{\sgn}{sign}

\newcommand{\iphi}{\phi^-}

\def\Rset{\mathbb{R}}

\newcommand{\PP}{{\mathbb P}}
\newcommand{\QQ}{{\mathbb Q}}

\newcommand{\one}{{\mathbf{1}}}
\newcommand{\zero}{{\mathbf{0}}}
\newcommand{\Wball}[1]{{\cB^\infty_{#1}}}

\newcommand{\ov}{\overline}

\newcommand{\e}{\epsilon}

\newtheorem{theorem}{Theorem}
\newtheorem*{informaltheorem}{Informal Theorem}
\newtheorem{definition}{Definition}
\newtheorem{lemma}{Lemma}
\newtheorem{corollary}{Corollary}
\newtheorem{prop}{Proposition}
\newtheorem{assumption}{Assumption}

%% file: Sections/1-Introduction.tex
\begin{abstract}
	We study the consistency of surrogate risks for robust binary classification.
	It is common to learn robust classifiers by adversarial training, which seeks to minimize the expected $0$-$1$ loss when each example can be maliciously corrupted within a small ball.
	We give a simple and complete characterization of the set of surrogate loss functions that are \emph{consistent}, i.e., that can replace the $0$-$1$ loss without affecting the minimizing sequences of the original adversarial risk, for any data distribution.
	We also prove a quantitative version of adversarial consistency for the $\rho$-margin loss.
	Our results reveal that the class of adversarially consistent surrogates is substantially smaller than in the standard setting, where many common surrogates are known to be consistent.
%	Furthermore, We provide a quantitative version of our consistency result for the common $\rho$-margin loss.
%	
%	In standard binary classification, there are well known conditions under which minimizing a surrogate risk gives rise to the same classifiers as would be obtained by minimizing the true classification risk, which justifies the wide use of such surrogate risks for training classifiers in practice.
%	However, these results do not carry over to adversarial training procedures, where each example can be corrupted by a small perturbation.
%	
%	
%	but these results do not carry over to binary classification problems with adversarial perturbations.
%	
%	
%     Robustness to adversarial perturbations is of paramount concern in modern machine learning. One of the state-of-the-art methods for training robust classifiers is adversarial training, which involves minimizing a supremum-based surrogate risk. The statistical consistency of surrogate risks is well understood in the context of standard machine learning, but not in the adversarial setting. In this paper, we characterize which supremum-based risks are consistent for all distributions in the context of binary classification. Furthermore, we prove a quantitative surrogate bound for the adversarial $\rho$-margin loss.
%   
\end{abstract}
\section{Introduction}\label{sec:intro}
        A central issue in the study of neural nets is their susceptibility to adversarial perturbations---perturbations imperceptible to the human eye can cause a neural net to misclassify an image \citep{szegedy2013intriguing,biggio2013evasion}. The same phenomenon appears in other types of data such as speech and text. As deep nets are used in applications such as self-driving cars and medical imaging \citep{paschali2018generalizability,LiXuTraffic}, training classifiers robust to adversarial perturbations is a central question in machine learning.

    The foundational theory of surrogates for classication in well understood. In the standard classification setting, one seeks to minimize the \emph{classification} risk--- the proportion of incorrectly classified data. 
    %this risk incurs a penalty of 1 if a data point is classified incorrectly and otherwise no penalty is incurred. 
    Since minimizing the classification risk is typically computationally intractable \citep{BenDavidEironLong2003}, a common approach is to instead minimize a better-behaved alternative called the \emph{surrogate risk}. However, one must verify that classifiers with low surrogate risk also achieve low classification risk. If for every data distribution, a sequence of functions minimizing the surrogate also minimizes the classification risk, the surrogate risk is called \emph{consistent}. Many classic papers study the consistency of surrogate risks in the standard classification setting \citep{BartlettJordanMcAuliffe2006,Lin2004,Steinwart2007,LongServedioH-consistency,ZhangAgarwal}. 
    
    %Furthermore, one would hope that minimizing the surrogate risk would efficiently minimize the classification risk as well. To understand this aspect of surrogate risks, prior work \citep{BartlettJordanMcAuliffe2006,ReidWilliamson2009,Steinwart2007} compute surrogate risk bounds.  %of these results rely on the fact that surrogate risks are the integral of a surrogate loss whose value only relies only on a single point $x$.
    
    %In the adversarial scenario, obtaining a robust classifier requires minimizing  the \emph{adversarial} classification risk, which incurs a penalty at a point $x$ if it is $\e$-close to lying in the opposite class. %can be perturbed into the opposite class by a perturbation in an $\e$-ball.
    %As in the standard case, minimizing this risk directly is computationally intractable, so again one minimizes a surrogate. One common choice is a procedure called \emph{adversarial training}, in which the surrogate risk involves taking the supremum of a surrogate loss function over an $\e$-ball. State-of-the-art methods for training robust classifiers involve minimizing such supremum-based surrogates \citep{ShafahiNajibietal19,kurakin2017adversarial,madry2019deep}. 
    
    Unlike the standard case, however, little is known about the consistency of surrogate risks in the context of adversarial training, which involves risks that compute the supremum of a surrogate loss function over an $\e$-ball.
%    Adversarial training \citep{ShafahiNajibietal19,kurakin2017adversarial,madry2019deep}, a popular adversarial learning algorithm,  involves surrogates risks that compute the supremum of a surrogate loss function over an $\e$-ball.
    Though this question has been partially studied in the literature~\citep{AwasthiFrankMao2021,AwasthiMaoMohri,MeunierEttedguietal22}, a general theory is lacking.
    Existing results reveal, however, that the situation is substantially different from the standard case: for instance, \citep{MeunierEttedguietal22} show that no %supremum-based
    \emph{convex} surrogate can be adversarially consistent.
	To our knowledge, no adversarially consistent risks are known.

	In this work, we give a complete characterization of adversarial consistency for surrogate losses.
    
    %This paper aims to answer the question: when are supremum-based risks adversarially consistent? 

    \textbf{Our Contributions:}
    \begin{itemize}
        \item In Section~\ref{sec:adversarially_consistent_losses} we give a surprisingly simple necessary and sufficient condition for adversarial consistency:% of adversarially consistent surrogate risks under the evasion attack:
    \begin{informaltheorem}
    Under reasonable assumptions on the surrogate loss $\phi$, the supremum-based $\phi$-risk is adversarially consistent %when learning over the class of all measurable functions
    if and only if
        %Consider binary classification with labels $\{-1,+1\}$ and risk $R_\phi^\e(f)=\E \sup_{\|\bx-\bx'\|\leq \e} \phi(yf(\bx'))$. Under reasonable assumptions on the surrogate loss $\phi$, the risk $R_\phi^\e$ is adversarially consistent when learning over the class of all measurable functions if and only if
     $\inf_\alpha \phi(\alpha)/2+ \phi(-\alpha)/2<\phi(0)$.
    \end{informaltheorem}\label{ith:1}
In particular, this result proves consistency for any loss function that is \emph{not} midpoint convex at the origin.
\item In Section~\ref{sec:quantitative}, we specialize to the case of the $\rho$-margin loss, where we obtain a quantitative proof of adversarial consistency by explicitly bounding the excess adversarial risk.
%we show an explicit bound relating the excess adversarial surrogate error and the excess adversarial classification error
    \end{itemize}

    %The value of a surrogate loss at a point $x$ depends on all the points in an $\e$-ball around $x$. As a result, one cannot directly apply the techniques for studying consistency in the standard setting to analyze adversarial surrogate risks. In order to apply these techniques, our approach uses recent results on the theory of adversarial learning to reduce the problem to consistency in the non-adversarial context. 
    %We leverage recent results on the theory of adversarial learning and results of \citep{MeunierEttedguietal22} to give a surprisingly simple and general characterization:
    %Several existing papers study consistency in the adversarial scenario but overall results are sparse \citep{bhattacharjee2020nonparametric,bhattacharjee2021consistent,AwasthiFrankMao2021,AwasthiMaoMohri}. 

    To the best of the authors' knowledge, this paper is the first to prove that a loss-based learning procedure is consistent for a wide range of distributions in the adversarial setting. 
    As mentioned above, the $\rho$-margin loss $\phi_\rho(\alpha)=\min(1,\max( 1-\alpha/\rho,0))$ satisfies the conditions of Informal Theorem above, as does the shifted sigmoid loss %proposed by
     $\phi_\tau(\alpha)=1/(1+\exp(\alpha-\tau))$ with $\tau>0$, which confirms a conjecture of  \citet{MeunierEttedguietal22}.
    By contrast, all convex losses satisfy $\inf_\alpha \phi(\alpha)/2+ \phi(-\alpha)/2=\phi(0)$, and are therefore not adversarially consistent.% Some losses that satisfy the property of Informal Theorem~\ref{ith:1} are the $\rho$-margin loss $\phi_\rho(\alpha)=\min(1,\max( 1-\alpha/\rho,0))$ and the the shifted sigmoid loss proposed by \citep{MeunierEttedguietal22}, $\phi(\alpha)=1/(1+\exp(\alpha-\tau))$, $\tau>0$. 

    In addition to consistency, one would hope to obtain a quantitative comparison between the adversarial surrogate risk and the adversarial classification risk. %the surrogate  minimizing an adversarial surrogate would efficiently minimize the adversarial classification risk as well.
    Our bound in Section~\ref{sec:quantitative} shows that the excess error of the adversarial $\rho$-margin loss is a linear upper bound on the adversarial classification error, which implies that minimizing the adversarial $\rho$-margin loss is an effective procedure for minimizing the adversarial classification error. Extending the bound in Section~\ref{sec:quantitative} to further losses remains an open question.

%% file: Sections/2-Related_works.tex
    \section{Related Works}\label{sec:related_works}
        Many previous works have studied the consistency of surrogate risks \citep{BartlettJordanMcAuliffe2006,Lin2004,Steinwart2007, LongServedioH-consistency,ZhangAgarwal}. The classic papers by \citep{BartlettJordanMcAuliffe2006,Lin2004,zhang04} explore the consistency of surrogate risks over all measurable functions. The works \citep{LongServedioH-consistency,ZhangAgarwal,AwasthiMaoMohriZhong22} study $\cH$-consistency, which is consistency restricted to a smaller set of functions. \citet{Steinwart2007} generalizes some of these results into a framework referred to as \emph{calibration}.  \citet{AwasthiFrankMao2021,bao2021calibrated,AwasthiMaoMohri,MeunierEttedguietal22} then use this framework to analyze the calibration of adversarial surrogate losses.  Furthermore \citet{MeunierEttedguietal22} relate calibration to consistency for adversarial losses in certain cases --- they show that no convex loss is adversarially consistent. They also conjecture that a class of surrogate losses called the \emph{odd shifted} losses are adversarially consistent. \citet{MeunierEttedguietal22} also show that in a restricted setting, surrogates are consistent for `optimal attacks'. The proof of our result formalizes this intuition. Simultaneous work \citep{MaoMohriZhong2023crossentropy} shows that the $\rho$-margin loss is adversarially $\cH$-consistent for typical function classes. Lastly, \citet{bhattacharjee2020nonparametric,bhattacharjee2021consistent} use a different set of techniques to study the consistency of non-parametric methods in adversarial scenarios. 
        
        Our results rely on recent works establishing  the properties of minimizers to surrogate adversarial risks. \citep{AwasthiFrankMohri2021, PydiJog2021,BungertGarciaMurray2021} all proved the existence of minimizers to the adversarial risk and \citep{PydiJog2021} proved a minimax theorem for the zero-one loss.  Building on the work of \citep{PydiJog2021}, \citep{FrankNilesWeed23minimax} later proved similar existence and minimax statements for arbitrary surrogate losses. \citet{TrillosJacobsKim22,trillos2023existence} extend some of these results to the multiclass case. Lastly, \citep{trillosMurray2020} %TODO add citation to Frank2022} 
        study further properties of the minimizers to the adversarial classification loss.

%        While this paper focuses on consistency, there have been several other directions studying adversarial learning from a theoretical perspective. The papers \citep{bubeck2018adversarial, bubeck2018adversarial2, nakkiran2019adversarial, degwekar2019computational, tsipras2018robustness} explore statistical and computational bottlenecks in adversarial learning. There have been several attempts to guarantee the performance of adversarial learning algorithms. \citet{Raghunathan2018,WengZhangChenSong2018,ZhangWengChenHsieh2018,WongKolter2018,sinha2020certifying} approach this problem by certifying certificates to robustness. Furthermore, several works study the sample complexity in the adversarial setting; \citep{khim2018adversarial, YinRamchandranBartlett2019, awasthi2020adversarial, MontasserHannekeSrebro19} study the Rademacher complexity and VC-dimension of adversarial learning while \citep{XingSongCheng2021} uses algorithmic stability to give an upper bound on the generalization of adversarial training.

%% file: Sections/3-background.tex
\section{Problem Setup}\label{sec:problem_setup}
This section contains the necessary background for our results.
Section~\ref{sec:surrogate_risks} gives precise definitions for the main concepts, and Section~\ref{sec:existence_minimax} describes the minimax theorems that are at the heart of our proof.
\subsection{Surrogate Risks}\label{sec:surrogate_risks}
This paper studies binary classification on $\Rset^d$. Explicitly, labels are $\{-1,+1\}$ and the data is distributed according to a distribution $\cD$ on the set $\Rset^d\times\{-1,+1\}$. The measures $\PP_1$, $\PP_0$ define the relative probabilities of finding points with a given label in a region of $\Rset^d$. Formally, define measures on $\Rset^d$ by
\[\PP_1(A)=\cD(A\times \{+1\}), \PP_0(A)=\cD(A\times \{-1\}).\]
%Throughout this paper, we assume that 
%\begin{assumption}\label{as:absolute_cont}
%    $\PP_1,\PP_0$ are absolutely continuous with respect to Lebesgue measure.
%\end{assumption}

 The \emph{classification risk} $R(f)$ is then the probability of misclassifying a point under $\cD$:
\begin{equation}\label{eq:standard_zero_one_loss}
    R(f)= \int\one_{f(\bx)\leq 0}d\PP_1+\int  \one_{f(\bx)> 0} d\PP_0.    
\end{equation}

The surrogate to $R$ is
\begin{equation}\label{eq:standard_phi_loss}
    R_\phi(f)=\int \phi(f)d\PP_1+\int\phi(-f) d\PP_0\,.
\end{equation} 
A classifier can be obtained by minimizing either $R$ or $R_\phi$ over the set of all measurable functions.
%optimized over all measurable functions. 
A point $\bx$ is then classified according to $\sgn f$. There are many possible choices for $\phi$---typically one chooses a loss that is easy to optimize. In this paper, we assume that 
\begin{assumption}\label{as:phi}
    $\phi$ is non-increasing, non-negative, continuous, and $\lim_{\alpha\to \infty}\phi(\alpha)=0$.
\end{assumption}
 Most surrogate losses in machine learning satisfy this assumption. Learning algorithms typically optimize the risk in \eqref{eq:standard_phi_loss} using an iterative procedure, which produces a sequence of functions that minimizes $R_\phi$. We call $R_\phi$ a \emph{consistent risk} and $\phi$ a \emph{consistent loss} if for all distributions, every minimizing sequence of $R_\phi$ is also a minimizing sequence of $R$.\footnote{In the context of standard (non-adversarial) learning, the concept we defined as consistency is often referred to as \emph{calibration}, see for instance \citep{BartlettJordanMcAuliffe2006,Steinwart2007}. We opt for the term `consistency' as the prior works \citep{AwasthiFrankMao2021,AwasthiMaoMohri,MeunierEttedguietal22} use calibration to refer to a different but related concept in the adversarial setting.} Alternatively, the risks $R$, $R_\phi$ can be expressed in terms of the quantities $\PP=\PP_0+\PP_1$ and $\eta=d\PP_1/d\PP$. For all $\eta\in [0,1]$, define

  \begin{equation}\label{eq:C_def}
    C(\eta,\alpha)=\eta \one_{\alpha \leq 0}+(1-\eta) \one_{\alpha >0},\quad C^*(\eta)=\inf_\alpha C(\eta,\alpha),
 \end{equation}
 \begin{equation}\label{eq:C_phi_def}
    C_\phi(\eta,\alpha) = \eta\phi(\alpha)+(1-\eta)\phi(-\alpha),\quad C_\phi^*(\eta)=\inf_\alpha C_\phi(\eta,\alpha) 
\end{equation}
For more on the definitions of $R,R_\phi,C,C_\phi$, see \citep{BartlettJordanMcAuliffe2006} or Sections~3.1 and~3.2 of \citep{FrankNilesWeed23minimax}. Using these definitions, $R(f)=\int C(\eta(\bx),f(\bx))d\PP$ and 
 \begin{equation}\label{eq:pw_loss}
    R_\phi(f)=\int C_\phi(\eta(\bx),f(\bx))d\PP    
\end{equation}
%Furthermore, the optimal values of $R,R_\phi$ are $inf_f R(f)=\int C^*(\eta)d\PP$ and $\inf_f R_\phi(f)=\int C_\phi^*(\eta)d\PP$, see Lemma~\ref{lemma:R_phi_min} of Appendix~\ref{app:alt_consistency_characterization} for a proof sketch.

This alternative view of the risks $R$ and $R_\phi$ provides a `pointwise' criterion for consistency--- if the function $f(\bx)$ minimizes $C_\phi(\eta(\bx),\cdot)$ at each point, then it also minimizes $R_\phi$. However, minimizers to $C_\phi(\eta,\cdot)$ over $\Rset$ do not always exist--- consider for instance $\eta=1$ for the exponential loss $\phi(\alpha)=e^{-\alpha}$. In general, for minimizers of $C_\phi(\eta,\cdot)$ to exist, one must work over the extended real numbers $\ov \Rset=\Rset \cup \{-\infty,+\infty\}$. The following proposition proved in Appendix~\ref{app:alt_consistency_characterization} implies that `pointwise' considerations also extends to minimizing sequences of functions.
    \begin{prop}\label{prop:alt_consistency_characterization} %A loss $\phi$ is consistent iff every minimizer of $R_\phi$ is also a minimizer of $R$. 
    The following are equivalent:
    \begin{enumerate}[label=\arabic*)]
        \item\label{it:consistent} $\phi$ is consistent
        \item \label{it:C_seq_background}Every minimizing sequence of $C_\phi(\eta,\cdot)$ is also a minimizing sequence of $C(\eta,\cdot)$
        \item \label{it:R_min_background}Every $\ov \Rset$-valued minimizer of $R_\phi$ is a minimizer of $R$
    \end{enumerate}
    \end{prop}
    This result is well-known in prior literature; in particular the equivalence between \ref{it:C_seq_background} and \ref{it:R_min_background} is closely related to the equivalence between calibration and consistency in the non-adversarial setting \citep{Steinwart2007}.
    Most importantly, the equivalence between \ref{it:consistent} and \ref{it:R_min_background} reduces studying minimizing sequences of functionals to studying minimizers of functions. We will show that the equivalence between  \ref{it:consistent}  and \ref{it:C_seq_background} has an analog in the adversarial scenario, but the equivalence between  \ref{it:consistent}  and \ref{it:R_min_background} does not.

In the adversarial classification setting, every $x$-value is perturbed by a malicious adversary before undergoing classification by $f$. We assume that these perturbations are bounded by $\e$ in some norm $\|\cdot\|$ and furthermore, the adversary knows both our classifier $f$ and the true label of the point $\bx$.
In other words, $f$ misclassifies $(\bx,y)$ when there is a point $\bx'\in \ov{B_\e(\bx)}$ for which $\one_{f(\bx')\leq 0} =1$ for $y=+1$ and $\one_{f(\bx')>0}=1$ for $y=-1$. 
Conveniently, this criterion can be expressed in terms of suprema. 
For any function $g$, we define
\begin{equation*}
    S_\e(g)(\bx)=\sup_{\|\bh\|\leq \e} g(\bx+\bh)
\end{equation*}
A point $\bx$ with label $+1$ is misclassified when $S_\e(\one_{f\leq 0})(\bx)=1$ and a point $\bx$ with label $-1$ is misclassified when $S_\e(\one_{f> 0})(\bx)=1$. Hence the expected fraction of errors under the adversarial attack is
\begin{align}
    R^\e(f)&= \int S_\e(\one_{f\leq 0})d\PP_1+\int  S_\e(\one_{f>0}) d\PP_0, \label{eq:adv_zero_one_loss_sup}
\end{align}
which is called the \emph{adversarial classification risk} \footnote{Defining this integral requires some care because for a Borel function $g$, $S_\e(g)$ may not be measurable; see Section~3.3 and Appendix~A of \citep{FrankNilesWeed23minimax} for details.}. Again, optimizing the empirical version of \eqref{eq:adv_zero_one_loss_sup} is computationally intractable so instead one minimizes a surrogate of the form
\begin{align}
    &R_\phi^\e(f)=\int S_\e( \phi\circ f)d\PP_1 +\int S_\e( \phi\circ -f)d\PP_0\label{eq:adv_phi_loss}
\end{align}
Due to the supremum in this expression, we refer to such a risk as a \emph{supremum-based surrogate}. We define adversarial consistency as 
\begin{definition}
    The risk $R_\phi^\e$ is \emph{adversarially consistent} if for every data distribution, every sequence $f_n$ which minimizes $R_\phi^\e$ over all Borel measurable functions also minimizes $R^\e$. We say that the loss $\phi$ is \emph{adversarially consistent} if the risk $R_\phi^\e$ is adversarially consistent.
\end{definition}

Many convex and non-convex losses are consistent in standard classification \citep{BartlettJordanMcAuliffe2006,zhang04,Steinwart2007,ReidWilliamson2009,Lin2004}. By contrast, adversarial consistency often fails. For instance, 
    \citet{MeunierEttedguietal22} show that convex losses are not adversarially consistent. Furthermore, their example shows that the equivalence between \ref{it:consistent} and \ref{it:R_min_background} in Proposition~\ref{prop:alt_consistency_characterization} does \emph{not} hold in the adversarial context. Thus, to understand adversarial consistency, it does not suffice to compare minimizers of $R_\phi^\e$ and $R^\e$.
    To illustrate this distinction, we show the following result, adapted from \citep{MeunierEttedguietal22}.
    
    \begin{prop}\label{prop:convex_consistency_counterexample}
        Assume that $\inf_\alpha \phi(\alpha)/2 + \phi(-\alpha)/2 = \phi(0)$. Then $\phi$ is not adversarially consistent.
    \end{prop}
%        This proposition is a minor adaptation of the counterexample from \citep{MeunierEttedguietal22}.
    \begin{proof}
        Let $\PP_0=\PP_1$ be the the uniform distribution on the ball $\ov{B_R(\zero)}$ and let $\e=2R$. Let $\phi$ be a loss function for which $\inf_\alpha \phi(\alpha)/2 + \phi(-\alpha)/2 = C_\phi^*(1/2)=\phi(0)$. Notice that $\inf_{f} R^\e(f)\geq \inf_f R(f)$ and $\inf_{f} R_\phi^\e(f)\geq \inf_f R_\phi(f)$. 
        Since $\PP_0 = \PP_1$, the optimal non-adversarial risk is $\inf_f R(f) = 1/2$.
        Moreover, as $C_\phi^*(1/2)=\phi(0)$, the optimal non-adversarial surrogate risk is $\inf_f R_\phi(f)= C_\phi^*(1/2)=\phi(0)$. Thus, for the function $f^*\equiv 0$, $R^\e(f^*)=\inf_f R(f)=1/2$ and $R_\phi^\e(f^*)=\inf_f R_\phi(f)=\phi(0)$. Therefore $f^*$ minimizes both $R_\phi^\e$ and $R^\e$. Now consider the sequence of functions
    \[f_n(\bx)=\begin{cases} \frac 1n  &\bx=0\\
    -\frac 1n &\bx\neq 0\end{cases}\]
     Because $\e=2R$, every point in the support of the distribution can be perturbed to every other point. Thus $S_\e(\phi\circ f_n)(\bx) = \phi(-1/n)$ and $S_\e(\phi \circ -f_n)(\bx)=\phi(-1/n)$. 
     However, $S_\e(\one_{f\leq 0})=1$ and $S_\e(\one_{f>0})=1$. Therefore, $R_\phi^\e(f_n)=\phi(-1/n)$ while $R^\e(f_n)=1$ for all $n$. As $\phi$ is continuous, $\lim_{n\to \infty} R_\phi^\e(f_n)=\phi(0)$. Thus $f_n$ is a minimizing sequence of $R_\phi^\e$ but not of $R^\e$, so $\phi$ is not adversarially consistent.
    \end{proof}

	This example shows that if $C_\phi^*(1/2)=\phi(0)$, then $\phi$ is not adversarially consistent. The main result of this paper is that this is the \emph{only} obstruction to adversarial consistency: $\phi$ is adversarially consistent if and only if $C_\phi^*(1/2)<\phi(0)$.

    We begin by showing that this condition suffices for consistency in the \emph{non-adversarial} setting. Surprisingly, despite the wealth of work on this topic, this condition does not appear to be known.  
    \begin{prop}\label{prop:a_def_consistent}
        If $C_\phi^*(1/2)<\phi(0)$, then $\phi$ is consistent.
    \end{prop}
    See Appendix~\ref{app:further_properties_of_losses} for a proof.

    Again, some losses that satisfy this property are the $\rho$-margin loss $\phi_\rho(\alpha)=\min(1,\max( 1-\alpha/\rho,0))$ and the the shifted sigmoid loss proposed by \citet{MeunierEttedguietal22}, $\phi(\alpha)=1/(1+\exp(\alpha-\tau))$, $\tau>0$. 
    (In fact, one can show that the class of shifted odd losses proposed by \citet{MeunierEttedguietal22} satisfy $C_\phi^*(1/2)<\phi(0)$.)

    Notice that all convex losses satisfy $C_\phi^*(1/2)=\phi(0)$:
    \[C_\phi^*(1/2)=\inf_\alpha \frac 12 \phi(\alpha) +\frac 12 \phi(-\alpha)\geq \phi(0)\]
    The opposite inequality follows from the observation that $C_\phi^*(1/2)\leq C_\phi(1/2,0)=\phi(0)$. In contrast, recall that a convex loss $\phi$ with $\phi'(0)<0$ is consistent \citep{BartlettJordanMcAuliffe2006}. 

    As conjectured by prior work \cite{bao2021calibrated,MeunierEttedguietal22}, the fundamental reason losses with $C_\phi^*(1/2)<\phi(0)$ are adversarially consistent is that minimizers of $C_\phi(\eta,\cdot)$ are uniformly bounded away from 0 for all $\eta$:
    \begin{lemma}\label{lemma:a_def_main}   
        The loss $\phi$ satisfies $C_\phi^*(1/2)<\phi(0)$ iff there is an $a>0$ for which any minimizer $\alpha^*$ of $C_\phi(\eta,\cdot)$ satisfies $|\alpha|\geq a$.   
    \end{lemma}

    See \ref{app:further_properties_of_losses} for a proof. Concretely, one can show that for the $\rho$-margin loss $\phi_\rho$, a minimizer $\alpha^*$ of $C_{\phi_\rho}(\eta,\cdot)$ must satisfy $|\alpha^*|\geq \rho$. Similarly, a minimizer $\alpha^*$ of $C_{\phi_\tau}(\eta,\cdot)$ of the shifted sigmoid loss $\phi_\tau=1/(1+\exp(\alpha-\tau))$, $\tau>0$ is always either $-\infty$ or $+\infty$.
    In \ref{sec:adversarially_consistent_losses}, we use this property to show that minimizing sequences of $R_\phi^\e$ must be uniformly bounded away from zero, thus ruling out the counterexample presented in Proposition~\ref{prop:convex_consistency_counterexample}.

\subsection{Minimax Theorems for Adversarial Risks}\label{sec:existence_minimax}
    We study the consistency of $\phi$ by by comparing minimizing sequences of $R^\e_\phi$ with those of $R^\e$. 
        In the next section, in order to compare these minimizing sequences, we will attempt to re-write the adversarial loss in a `pointwise' manner similar to %\eqref{eq:pw_loss}. 
        Proposition~\ref{prop:alt_consistency_characterization}. In order to achieve this representation of the adversarial loss, we apply minimax and complimentary slackness theorems from \citep{PydiJog2021,FrankNilesWeed23minimax}. 
        
        Before presenting these results, we introduce the $\infty$-Wasserstein metric from optimal transport. 
        For two finite probability measures $\QQ,\QQ'$ satisfying $\QQ(\Rset^d)=\QQ'(\Rset^d)$, let $\Pi(\QQ,\QQ')$ be the set of \emph{couplings} between $\QQ$ and $\QQ'$:
        \[\Pi(\QQ,\QQ')=\{ \gamma: \text{measure on }\Rset^d\times \Rset^d\text{ with } \gamma(A\times \Rset^d)=\QQ(A), \gamma(\Rset^d\times A)=\QQ'(A)\}\]
        
        The distance between $\QQ'$ and $\QQ$ in the Wasserstein $\infty$-metric $W_\infty$ is defined as
        \[W_\infty(\QQ,\QQ')=\inf_{\gamma \in \Pi(\QQ,\QQ')} \esssup_{(\bx,\by)\sim \gamma} \|\bx-\by\|. \]
        The $W_\infty$ distance is in fact a metric on the space of measures. We denote the $\infty$-Wasserstein ball around a measure $\QQ$ by 
        \[\Wball \e (\QQ)=\{\QQ'\colon \QQ' \text{ Borel},\quad W_\infty(\QQ,\QQ')\leq \e \}\]
        Informally, the measure $\QQ'$ is in $\Wball \e (\QQ)$ if perturbing points by at most $\e$ under the measure $\QQ$ can produce $\QQ'$. As a result, Wasserstein $\infty$-balls are fairly useful for modeling adversarial attacks. Specifically, one can show:
            \begin{lemma}\label{lemma:S_e_inequality}
	           For any function $g$ and measures $\QQ'$, $\QQ$ with $W_\infty(\QQ',\QQ)\leq \e$, the inequality $\int S_\e(g)d\QQ\geq \int gd\QQ'$ holds.
            \end{lemma}
        
        See Appendix~\ref{app:W_infty} for a proof. %(In fact, one can show that this inequality is an equality. See for instance Lemma~3 \citep{FrankNilesWeed23minimax}.)
        
        Minimax theorems from prior work use this framework to introduce dual problems to the adversarial classification risks \eqref{eq:adv_zero_one_loss_sup} and \eqref{eq:adv_phi_loss}. Let $\PP_0',\PP_1'$ be finite Borel measures and define
        \begin{equation}\label{eq:cdl_def}
            \cdl(\PP_0',\PP_1')=\int C^*\left(\frac{d \PP_1'}{d(\PP_0'+\PP_1')}\right)d(\PP_0'+\PP_1')
        \end{equation}
        where $C^*$ is defined by \eqref{eq:C_def}. The next theorem states that maximizing $\cdl$ over $W_\infty$ balls is in fact a dual problem to minimizing $\cprm$.
    \begin{theorem}\label{th:minimax_classification}
        Let $\cdl$ be defined by \eqref{eq:cdl_def}. 
	        \begin{equation}\label{eq:minimax_classification}
	           \inf_{\substack{f\text{ Borel}\\  \text{ $ \Rset$-valued}}}\cprm(f)=\sup_{\substack{\PP_0'\in\Wball \e (\PP_0)\\ \PP_1'\in \Wball \e (\PP_1)}}\cdl(\PP_0',\PP_1')
	        \end{equation}
	        and furthermore equality is attained for some Borel measurable $\hat f$ and 
         $\hat \PP_1,\hat \PP_0$ with $W_\infty(\hat \PP_0,\PP_0)\leq \e$ and $W_\infty(\hat \PP_1,\PP_1)\leq \e$.   
    \end{theorem}
        The first to show such a theorem was \citet{PydiJog2021}. In comparison to their Theorem 8, Theorem~\ref{th:minimax_classification} removes the assumption that $\PP_0,\PP_1$ are absolutely continuous with respect to Lebesgue measure and shows that the minimizer $\hat f$ is in fact Borel. We prove this theorem in Appendix~\ref{app:minimax_classification_proof}.
    \citet{FrankNilesWeed23minimax} prove a similar statement for the surrogate risk $R_\phi^\e$.
    This time, the dual objective is
    \begin{equation}\label{eq:dl_def}
        \dl(\PP_0',\PP_1')=\int C_\phi^*\left(\frac{d\PP_1'}{d(\PP_0'+\PP_1')}\right) d(\PP_0'+\PP_1')
    \end{equation}
    with $C_\phi^*$ defined by \eqref{eq:C_phi_def}.
    
    \begin{theorem}\label{th:minimax}
    Assume that Assumption~\ref{as:phi} holds, and define $\dl$ by \eqref{eq:dl_def}. Then
	        
	        \begin{equation}\label{eq:minimax}
	            \inf_{\substack{f \text{ Borel,}\\ f\text{ $\Rset$-valued}}} \prm(f)=\sup_{\substack{\PP_0'\in\Wball \e (\PP_0)\\ \PP_1'\in \Wball \e (\PP_1)}}  \dl(\PP_0',\PP_1')
             %\inf_{f \text{ Borel}} \prm(f)=\sup_{\substack{\PP_0'\in\Wball \e (\PP_0)\\ \PP_1'\in \Wball \e (\PP_1)}}  \dl(\PP_0',\PP_1')
	        \end{equation}
	        and furthermore equality in the dual problem is attained for some  
         $\PP_1^*,\PP_0^*$ with $W_\infty(\PP_0^*,\PP_0)\leq \e$ and $W_\infty(\PP_1^*,\PP_1)\leq \e$.
	    \end{theorem}

        \citet{FrankNilesWeed23minimax} proved this statement in Theorem~6 but with the infimum taken over $\ov \Rset$-valued functions.
        %, and furthermore, $R_\phi^\e$ attains its infimum over this set.
       To extend the result to $\Rset$-valued functions as in Theorem~\ref{th:minimax}, we show that $\inf_{f\text{ Borel},f\text{ $\ov \Rset$-valued}} R_\phi^\e(f)= \inf_{f\text{ Borel},f\text{ $\Rset$-valued}} R_\phi^\e(f)$ in Appendix~\ref{app:Rset-approximation}.

    %See \citep{PydiJog2021} and \citep{FrankNilesWeed23minimax} for a further discussion of these results.

%% file: Sections/4-minimizing_sequences.tex
\section{Adversarially Consistent Losses}\label{sec:adversarially_consistent_losses}
This section contains our main results on adversarial consistency. In light of Proposition~\ref{prop:convex_consistency_counterexample}, our main task is to show that a loss satisfying $C_\phi^*(1/2) < \phi(0)$ is adversarially consistent.

At a high level, we will show that every minimizing sequence of $R_\phi^\e$ must also minimize $R^\e$.
However, directly analyzing minimizing sequences $\{f_n\}$ of $R_\phi^\e$ and $R^\e$ is challenging due to the supremums in the definitions of the adversarial risks.
We therefore develop alternate characterizations of minimizing sequences to both functionals, based on complimentary slackness conditions derived from the convex duality results of Section~3.2.
%Our main tool for analyzing adversarial consistency will be complementary slackness conditions that follow from the minimax results in Section~\ref{sec:existence_minimax}. 
However, unlike standard complementary slackness conditions well known from convex optimization, these theorems allow us to characterize minimizing sequences as well as minimizers.

\subsection{Approximate Complimentary Slackness}
We first state this slackness result for the surrogate case, due to \citet[Lemmas 16 and 26]{FrankNilesWeed23minimax} and Theorem~\ref{th:minimax}.
%\jnw{You may wish to employ a thin space (typeset backslash followed by comma) between the function being integrated and $d\PP_i$ to improve readability.}
\begin{prop}\label{prop:approx_complimentary_slackness_phi}
	Let $(\PP_0^*,\PP_1^*)$ be any maximizers of $\dl$ over $\Wball \e (\PP_i)$. 
	Define $\PP^*=\PP_0^*+\PP_1^*$, $\eta^*=d\PP_1^*/d\PP^*$.
 	If $f_n$ is a minimizing sequence for $\prm$, then the following hold:
  		\begin{equation}\label{eq:C_comp_slack_approx}
			\lim_{n\to \infty} \int C_\phi(\eta^*,f_n)d\PP^*= \int C_\phi^*(\eta^*)d\PP^*.
		\end{equation}
		\begin{equation}\label{eq:sup_comp_slack_approx}
			\lim_{n\to \infty} \int S_\e(\phi \circ f_n)d\PP_1-\int \phi \circ f_nd\PP_1^*	=0,\quad 	     \lim_{n\to \infty} \int S_\e(\phi \circ -f_n) d\PP_0- \int \phi \circ -f_nd\PP_0^*=0
		\end{equation}

\end{prop}
\begin{proof}
    Let $R^\e_{\phi,*}$ be the minimal value of $R_\phi^\e$ and choose a $\delta>0$. Then for sufficiently large $N$, $n\geq N$ implies that $R_\phi^\e(f_n)\leq R^\e_{\phi,*}+\delta$. Lemma~\ref{lemma:S_e_inequality} and the definition of $C_\phi^*$ in \eqref{eq:C_phi_def} further imply that  
    \begin{equation}\label{eq:approx_c_slack_1}
        R^\e_{\phi,*}+\delta \geq \int S_\e(\phi \circ f_n)d\PP_1+\int S_\e(\phi \circ -f_n) d\PP_0\geq \int \phi \circ f_n d\PP_1^*+\int \phi \circ -f_n d\PP_0^*\geq R^\e_{\phi,*}    
    \end{equation}
    
    As $R^\e_{\phi,*}= \int C_\phi^*(\eta^*)d\PP^*$, this relation immediately implies \eqref{eq:C_comp_slack_approx}. 
    
        Next, Lemma~\ref{lemma:S_e_inequality} again implies that 
        \begin{equation}\label{eq:S_e_ineq_consequence}
            \int S_\e(\phi \circ f_n)d\PP_1\geq \int \phi \circ f_nd\PP_1^* \quad\text{and}\quad \int S_\e(\phi \circ -f_n)d\PP_0\geq \int \phi \circ -f_n d\PP_0^*     
        \end{equation}
         while \eqref{eq:approx_c_slack_1} implies that 
        \[R_{\phi,*}^\e- \int \phi \circ f_n d\PP_1^*+\int \phi \circ -f_n d\PP_0^*\leq 0.\] 

        Therefore, subtracting $\int \phi \circ f_n d\PP_1^*+\int \phi \circ -f_n d\PP_0^*$ from \eqref{eq:approx_c_slack_1} results in 
    \begin{equation}\label{eq:approx_comp_slack_2}
        \delta \geq \left(\int S_\e(\phi \circ f_n)d\PP_1 - \int \phi \circ f_n d\PP_1^*\right)+\left( \int S_\e(\phi \circ -f_n) d\PP_0-\int \phi \circ -f_n d\PP_0^*\right)\geq 0.    
    \end{equation}

    Again, \eqref{eq:S_e_ineq_consequence} implies that the quantities on parentheses are both positive which implies \eqref{eq:sup_comp_slack_approx}.  

\end{proof}

Proposition~\ref{prop:approx_complimentary_slackness_phi} shows that minimizing sequences of $R_\phi^\e$ satisfy two properties: 1) The sequence $\{f_n\}$ must minimize the \emph{standard} $\phi$-risk $R_\phi$ with measures $\PP_0^*$, $\PP_1^*$ in place of $\PP_0,\PP_1$, 2) At the limit, the measures $\PP_0^*,\PP_1^*$ are best adversarial attacks on $\phi\circ f_n,\phi \circ -f_n$. In fact, one can show that $\{f_n\}$ is a minimizing sequence of $R_\phi^\e$ \emph{if and only if} it satisfies these properties. Crucially, a very similar characterization holds for minimizers of the adversarial classification loss.
We state and prove the `only if' direction of this characterization in Proposition~\ref{prop:approx_complimentary_slackness_classification}.

\begin{prop}\label{prop:approx_complimentary_slackness_classification}
    Let $f_n$ be a sequence and let $\PP_0^*$, $\PP_1^*$ be measures in $\Wball \e (\PP_i)$. Define $\PP^*=\PP_0^*+\PP_1^*$, $\eta^*=d\PP_1^*/d\PP^*$. If the following two conditions hold:
        \begin{equation}\label{eq:C_comp_slack_approx_classification}
        	\lim_{n\to \infty} \int C(\eta^*,f_n)d\PP^*=\int C^*(\eta^*)d\PP^*
    \end{equation}
        \begin{equation}\label{eq:sup_comp_slack_approx_classification}
			\lim_{n\to \infty} \int S_\e(\one_{f_n\leq 0})d\PP_1-\int  \one_{f_n\leq 0}d\PP_1^*=0,\quad           \lim_{n\to \infty} \int S_\e(\one_{f_n>0})d\PP_0-\int \one_{f_n>0}d\PP_0^*=0,\end{equation}
    then $f_n$ is a minimizing sequence of $R^\e$.
\end{prop}
\begin{proof}
        Equation~\ref{eq:C_comp_slack_approx_classification} implies that the limit $\lim_{n\to \infty} C(\eta^*,f_n)d\PP^*$ exists. Thus \eqref{eq:C_comp_slack_approx_classification} and \eqref{eq:sup_comp_slack_approx_classification} imply that
    \begin{align*}
        \lim_{n\to \infty} R^\e(f_n)&=\lim_{n\to \infty} \int S_\e(\one_{f_n\leq 0})d\PP_1+\int S_\e(\one_{f_n>0})d\PP_0
        =\lim_{n\to \infty} \int \one_{f_n\leq 0}d\PP_1^*+\int \one_{f_n>0}d\PP_0^*\\
        &=\lim_{n\to \infty}\int C(\eta^*,f_n)d\PP^* 
         =\int C^*(\eta^*)d\PP^* 
        = \bar R(\PP_0^*, \PP_1^*)\,.
    \end{align*}
    
    Therefore, 
    Strong duality (Theorem~\ref{th:minimax_classification}) then implies that 
    \[\lim_{n\to \infty} R^\e(f_n)\leq \sup_{\substack{\PP_0'\in\Wball \e (\PP_0)\\ \PP_1'\in \Wball \e (\PP_1)}}\cdl(\PP_0',\PP_1')=\inf_{\substack{f\text{ Borel}\\  \text{ $ \Rset$-valued}}}\cprm(f)
	        \]
    and therefore, $f_n$ is a minimizing sequence.
\end{proof}

We end this section by comparing the different criteria for consistency presented in Proposition~\ref{prop:alt_consistency_characterization} with Propositions~\ref{prop:approx_complimentary_slackness_phi} and~\ref{prop:approx_complimentary_slackness_classification}. Together, Propositions~\ref{prop:approx_complimentary_slackness_phi} and~\ref{prop:approx_complimentary_slackness_classification} will allow us to compare minimizing sequences of $\prm$ to those of $R^\e$ by showing that any sequence satisfying~\eqref{eq:C_comp_slack_approx}--\eqref{eq:sup_comp_slack_approx} must also satisfy~\eqref{eq:C_comp_slack_approx_classification}--\eqref{eq:sup_comp_slack_approx_classification}.
This statement is the analog to \ref{it:C_seq_background} of Proposition~\ref{prop:alt_consistency_characterization}.
Indeed, because $C_\phi(\eta^*,f_n)\geq C_\phi^*(\eta^*)$, \eqref{eq:C_comp_slack_approx} is actually equivalent to to $C_\phi(\eta^*,f_n)\to C_\phi^*(\eta^*)$ in $L^1(\PP^*)$. However, the extra criterion \eqref{eq:sup_comp_slack_approx_classification} implies an additional constraint on the structure of the minimizing sequence.
This additional constraint is the reason \ref{it:R_min_background} of Proposition~\ref{prop:alt_consistency_characterization} is false in the adversarial setting. In the restricted situation where $\dl = \bar R$, \citet{MeunierEttedguietal22} show that \eqref{eq:C_comp_slack_approx} implies \eqref{eq:C_comp_slack_approx_classification} (Proposition~4.2). However, this observation does not suffice to conclude consistency.

\subsection{Adversarial Consistency}
We are now in a position to prove consistency.
Before presenting the full proof, we pause to discuss the overall strategy.
Consistency will follow from three considerations. First, every minimizing sequence of $R_\phi^\e$ satisfies conditions \eqref{eq:C_comp_slack_approx} and \eqref{eq:sup_comp_slack_approx}.
Second, conditions \eqref{eq:C_comp_slack_approx} and \eqref{eq:sup_comp_slack_approx} imply the very similar conditions \eqref{eq:C_comp_slack_approx_classification} and \eqref{eq:sup_comp_slack_approx_classification}.
Finally, any function sequence satisfying  \eqref{eq:C_comp_slack_approx_classification} and \eqref{eq:sup_comp_slack_approx_classification} must be a minimizing sequence to $R^\e$.
The first and last steps are the content of Propositions~\ref{prop:approx_complimentary_slackness_phi} and~\ref{prop:approx_complimentary_slackness_classification}, so it remains to justify the middle step.

Verifying that \eqref{eq:C_comp_slack_approx} implies \eqref{eq:C_comp_slack_approx_classification} is straightforward. The relation \eqref{eq:C_comp_slack_approx} actually states that $f_n$ minimizes the \emph{standard} surrogate risk with respect to the distribution given by $\PP_0^*$, $\PP_1^*$. Therefore \eqref{eq:C_comp_slack_approx} implies \eqref{eq:C_comp_slack_approx_classification} so long as $\phi$ is consistent.

The main difficulty is verifying \eqref{eq:sup_comp_slack_approx_classification}, due to the discontinuity of $\one_{\alpha<0}$, $\one_{\alpha\geq 0}$ at 0. Due to this discontinuity, one cannot directly argue that \eqref{eq:sup_comp_slack_approx} implies \eqref{eq:sup_comp_slack_approx_classification}: to simplify the discussion, assume that $\phi$ is strictly decreasing on a neighborhood of the origin, in which case $\one_{\alpha<0}=\one_{\phi(\alpha)>\phi(0)}$ and $\one_{\alpha\geq 0}=\one_{\phi(-\alpha)\geq \phi(0)}$. Recall that according to \eqref{eq:sup_comp_slack_approx}, in the limit $n\to \infty$, $\PP_0^*,\PP_1^*$ are the strongest attack in $\Wball \e(\PP_0)\times \Wball \e(\PP_1)$, or informally, $ S_\e(\phi \circ f_n)(\bx)$ approaches $\phi (f_n(\bx'))$ for an optimal perturbation $\bx'$ w.h.p., with a similar condition for $\phi \circ -f_n$. However, due to the discontinuity of $\one_{\phi(-\alpha)\geq \phi(0)}$ at $\phi(0)$, if $f_n(\bx')\to 0$ as $n\to \infty$, this relation does not imply that $\one_{S_\e(\phi\circ -f_n)(\bx)\geq \phi(0)}$ approaches $\one_{\phi\circ-f_n(\bx')\geq 0}$.

Lemma~\ref{lemma:a_def_main} says that if $C_\phi^*(1/2)<\phi(0)$, minimizers of $C_\phi(\eta, \cdot)$ are uniformly bounded away from $0$.
This fact suggests that minimizing sequences will also be bounded away from the origin, which will allow us to avoid the discontinuity there. Concretely, we show:

\begin{lemma}\label{lemma:minimizing_seq}
    Let $C_\phi^*(1/2)<\phi(0)$. Then there is a $\delta>0$ and a $c>0$ with $\phi(c)<\phi(0)$ for which $\alpha\in [-c,c]$ implies $C_\phi(\eta,\alpha)\geq C_\phi^*(\eta)+\delta$, uniformly in $\eta$. Furthermore, for this value of $c$, if $\alpha>c$ then $\phi(\alpha)<\phi(c)$.
\end{lemma}
We prove this lemma in Appendix~\ref{app:further_properties_of_losses}. Because $C_\phi(\eta^*,f_n)\to C_\phi^*(\eta^*)$ in $L^1(\PP^*)$, Lemma~\ref{lemma:minimizing_seq} implies that 
\begin{equation}%\label{eq:L_1_consequence}
    \lim_{n\to \infty} \PP^*(f_n\in [-c,c])=0.
\end{equation}

This relation is the key fact that allows us to show that \eqref{eq:sup_comp_slack_approx} implies \eqref{eq:sup_comp_slack_approx_classification}. The condition $C_\phi^*(1/2)<\phi(0)$ is essential for this step of the argument. %Equation~\ref{eq:L_1_consequence} implies that minimizing sequences must be bounded away from 0 on an event of probability tending to 1, which allows us to circumvent the discontinuity at the origin.

Lastly, Lemma~\ref{lemma:S_e_inequality} implies that $\int S_\e(\one_{f_n\geq 0})d\PP_1\geq \int \one_{f_n\geq 0} d\PP_1^*$ and thus to validate  \eqref{eq:sup_comp_slack_approx_classification}, it suffices to verify the opposite inequality in the limit $n\to \infty$.
%We require the following fact from optimal transport:

%%We prove this result in 
%
%Furthermore, one does not need to show that the limits in \eqref{eq:sup_comp_slack_approx_classification} and \eqref{eq:sup_comp_slack_approx_classification} exist--- it suffices to show these inequalities with $\liminf$s and $\limsup$.

\begin{lemma}\label{lemma:classification_comp_slack_easy}
    Let $f_n$ be a sequence of functions and let $\PP_0^*\in \Wball \e(\PP_0)$, $\PP_1^*\in \Wball \e(\PP_1)$. The equation
       \begin{equation}\label{eq:sup_comp_slack_approx_classification_mod_1}
			\limsup_{n\to \infty} \int S_\e(\one_{f_n\leq 0})d\PP_1\leq \liminf_{n\to \infty}\int  \one_{f_n\leq 0}d\PP_1^*
   \end{equation}
        implies the first relation of \eqref{eq:sup_comp_slack_approx_classification} and
        \begin{equation}\label{eq:sup_comp_slack_approx_classification_mod_0}
			\limsup_{n\to \infty} \int S_\e(\one_{f_n>0})d\PP_0\leq \liminf_{n\to \infty}\int \one_{f_n>0}d\PP_0^*
   \end{equation}
    implies the second relation of \eqref{eq:sup_comp_slack_approx_classification}.
    
\end{lemma}
See Appendix~\ref{app:easy} for a proof. These considerations suffice to prove the main result of this paper:

\begin{theorem}
    The loss $\phi$ is adversarially consistent if and only if $C_\phi^*(1/2)<\phi(0)$.
\end{theorem}
\begin{proof}
    The `only if' portion of the statement is Proposition~\ref{prop:convex_consistency_counterexample}.
    
    To show the `if' statement, recall the standard analysis fact: $\lim_{n\to \infty} a_n=a$ iff for all subsequences $\{a_{n_j}\}$ of $\{a_n\}$, there is a further subsequence $a_{n_{j_k}}$  for which $\lim_{k\to \infty} a_{n_{j_k}}=a$. This result implies that to prove $\prm$ is consistent, it suffices to show that every minimizing sequence $f_n$ of $R_\phi^\e$ has a subsequence $f_{n_j}$ that minimizes $\cprm$.
    
    Let $f_n$ be a minimizing sequence of $\prm$. For convenience, pick a subsequence $f_{n_j}$ for which the limits $\lim_{j\to \infty} \int S_\e(\one_{f_{n_j}<0})d\PP_0$, $\lim_{j\to \infty} \int S_\e(\one_{f_{n_j}\geq 0})d\PP_1$ both exist. For notational clarity, we drop the $\empty_j$ subscript and denote this sequence as $f_n$.
    
    By Proposition~\ref{prop:approx_complimentary_slackness_phi}, the equations \eqref{eq:C_comp_slack_approx} and \eqref{eq:sup_comp_slack_approx} hold. We will argue that $f_n$ is in fact a minimizing sequence of $R^\e$ by verifying the conditions of Proposition~\ref{prop:approx_complimentary_slackness_classification}.
    
    First, the relation \eqref{eq:C_comp_slack_approx} states that the sequence $f_n$ minimizes the \emph{standard} $\phi$-risk for the distribution given by $\PP_0^*$ and $\PP_1^*$. As the loss $\phi$ is consistent by Proposition~\ref{prop:a_def_consistent}, the sequence $f_n$ must minimize the standard classification risk for the distribution $\PP_0^*,\PP_1^*$. This statement implies \eqref{eq:C_comp_slack_approx_classification}.
    Next we will argue that \eqref{eq:sup_comp_slack_approx_classification} holds.

    Let $c,\delta$ be as in Lemma~\ref{lemma:minimizing_seq}. Because $C_\phi(\eta^*,f_n)\geq C_\phi^*(\eta^*)$, \eqref{eq:C_comp_slack_approx} implies that $C_\phi(\eta^*,f_n)$ converges to $C_\phi^*(\eta^*)$ in $L^1$. However, $L^1$ convergence implies convergence in measure (see for instance Proposition~2.29 of \citep{folland}), and therefore $\lim_{n\to \infty} \PP^*\big( C_\phi(\eta^*,f_n)>C_\phi^*(\eta^*)+\delta\big)=0$. Lemma~\ref{lemma:minimizing_seq} then implies that for $i=0,1$ 
    \begin{equation}\label{eq:delta_measure_limit}
        \lim_{n\to \infty}\PP_i^*( f_n\in [-c,c]) =0.
    \end{equation}
 
    Next, because $\phi$ is non-increasing, $f\leq 0$ implies $\phi(f)\geq \phi(0)$ and thus $\one_{f\leq 0}\leq \one_{\phi \circ f\geq \phi(0)}$. Furthermore, as the function $\alpha \mapsto \one_{\alpha\geq 0}$ is monotone and upper semi-continuous,
    \begin{equation}\label{eq:PP_1_pre_limit}
        \int S_\e(\one_{f_n\leq 0})d\PP_1 \leq \int S_\e(\one_{\phi \circ f_n\geq \phi(0)})d\PP_1\leq \int \one_{S_\e(\phi \circ f_n)\geq \phi(0)}d\PP_1.
    \end{equation}
      Let $\gamma_i$ be a coupling between $\PP_i$ and $\PP_i^*$ for which $\esssup_{(\bx,\by)\sim \gamma_i} \|\bx-\by\|\leq \e$. Then the measure $\gamma_i$ is supported on $\Delta_\e=\{(\bx,\by)\colon \|\bx-\by\|\leq \e\}$. Furthermore, as $S_\e(\phi \circ f_n)(\bx)\geq \phi \circ f_n(\bx')$ everywhere on $\Delta_\e$, the relation $S_\e(\phi \circ f_n)(\bx)\geq \phi \circ f_n(\bx')$ actually holds $\gamma_1$-a.e. Therefore, \eqref{eq:sup_comp_slack_approx} actually implies that $S_\e(\phi \circ f_n)(\bx)-\phi\circ f_n(\bx')$ converges in $\gamma_1$-measure to 0. In particular, since $\phi(c) < \phi(0)$, $\lim_{n\to \infty} \gamma_1\big(S_\e(\phi\circ f_n)(\bx)-\phi(f_n(\bx'))\geq \phi(0)-\phi(c)\big)=0$ and thus $\lim_{n\to \infty} \gamma_1(S_\e(\phi\circ f_n)(\bx)\geq \phi(0) \cap \phi\circ f_n(\bx')<\phi(c))=0$. Therefore, 
      \begin{align*}
        &\liminf_{n\to \infty} \PP_1(S_\e(\phi\circ f_n)(\bx)\geq \phi(0))=\liminf_{n\to \infty} \gamma_1(S_\e(\phi\circ f_n)(\bx)\geq \phi(0)\cap \phi\circ f_n(\bx')\geq \phi(c))\\
        &\leq \liminf_{n\to\infty} \gamma_1(\phi\circ f_n(\bx')\geq \phi(c))=\liminf_{n\to \infty} \PP_1^*(\phi \circ f_n(\bx')\geq \phi(c))
      \end{align*}
    This calculation implies
    \begin{equation}\label{eq:PP_1_limit}
         \liminf_{n\to \infty} \int \one_{S_\e(\phi \circ f_n)(\bx)\geq \phi(0)}d\PP_1\leq \liminf_{n\to \infty} \int \one_{\phi\circ f_n(\bx')\geq \phi(c)} d\PP_1^*\leq  \liminf_{n\to \infty}\int \one_{ f_n\leq c} d\PP_1^*    
    \end{equation}

    The last inequality follows because Lemma~\ref{lemma:minimizing_seq} states that $\alpha >c$ implies $\phi(\alpha)<\phi(c)$ and therefore $\one_{\phi\circ f_n\geq \phi(c)}\leq \one_{ f_n\leq c}$. Equation \ref{eq:delta_measure_limit} then implies
    \begin{equation}\label{eq:iphi_lim}
         \liminf_{n\to \infty} \int\one_{\phi\circ f_n\geq \phi(0)} d\PP_1^*\leq \liminf_{n\to \infty} \int \one_{f_n \leq c}d\PP_1^*=\liminf_{n\to \infty} \int \one_{f_n\leq -c} d\PP_1^*.
    \end{equation}
    Recall that the sequence $f_n$ was chosen so that the limit $\lim_{n\to \infty} \int S_\e(\one_{f_n\leq 0}) d\PP_1$ exists. Combining this fact with  \eqref{eq:PP_1_pre_limit}, \eqref{eq:PP_1_limit}, and \eqref{eq:iphi_lim} results in    
    \begin{equation}\label{eq:scs_inequality_PP_1}
        \limsup_{n\to \infty}\int S_\e(\one_{f_n\leq 0})d\PP_1\leq \liminf_{n\to \infty}\int \one_{f_n\leq -c} d\PP_1^*\leq \liminf_{n\to\infty} \int \one_{f_n\leq 0} d\PP_1^*
    \end{equation}
    
     The first relation of \eqref{eq:sup_comp_slack_approx_classification} then follows from \eqref{eq:scs_inequality_PP_1} together with Lemma~\ref{lemma:classification_comp_slack_easy}.

    A similar argument implies the second relation of \eqref{eq:sup_comp_slack_approx_classification}. Because $\one_{f>0}=\one_{-f<0}\leq \one_{-f\leq 0}$, the same chain of inequalities as \eqref{eq:PP_1_pre_limit}, \eqref{eq:PP_1_limit}, and \eqref{eq:iphi_lim} implies that 
    \[\limsup_{n\to \infty}\int S_\e(\one_{f_n>0})d\PP_0\leq \limsup_{n\to \infty}\int S_\e(\one_{-f_n\leq 0})d\PP_0\leq \liminf_{n\to \infty}\int \one_{-f_n \leq -c} d\PP_0^*= \liminf_{n\to \infty} \int \one_{f_n\geq c} d\PP_0^*\]
    As $c>0$, it follows that $\limsup_{n\to \infty} \int S_\e(\one_{f_n>0})d\PP_0\leq \liminf_{n\to \infty} \int \one_{f_n> 0} d\PP_0^*$. Once again, the second expression of \eqref{eq:sup_comp_slack_approx_classification} follows from this relation and Lemma~\ref{lemma:classification_comp_slack_easy}.
\end{proof}

%% file: Sections/5-Towards_H_consistency.tex
    \section{Quantitative Bounds for the $\rho$-Margin Loss}\label{sec:quantitative}
   
    As discussed in the introduction, statistical consistency is not the only property one would want from a surrogate. Hopefully, minimizing a surrogate will also efficiently minimize the classification loss.
    \citet{BartlettJordanMcAuliffe2006,Steinwart2007,ReidWilliamson2009} prove bounds of the form $R(f)-R_*\leq G_\phi (R_\phi^*(f)-R_{\phi,*})$ for a function $G_\phi$ and $R_*=\inf_f R(f)$, $R_{\phi,*}=\inf_f R_\phi(f)$. The function $G_\phi$ is an upper bound on the rate of convergence of the classification risk in terms of the rate of convergence of the surrogate risk. One would hope that $G_\phi$ is not logarithmic, as such a bound could imply that reducing $R(f)-R_*$ by a quantity $\Delta$ could require an exponential change of $e^\Delta$ in $R_\phi(f)-R_{\phi,*}$. \citet{BartlettJordanMcAuliffe2006} compute such $G_\phi$ for several popular losses in the standard classification setting. For example, they show the bounds $G_\phi(\theta)=\theta$ for the hinge loss $\phi(\alpha)=(1-\alpha)_+$ and $G_\phi(\theta)= \sqrt{\theta}$ for the squared hinge loss $\phi(\alpha)=(1-\alpha)_+^2$. On can prove an analogous bound for the $\rho$-margin loss in the adversarial setting:

        \begin{theorem}\label{th:tight_bound}
            		 Let $\phi_\rho=\min(1,\max( 1-\alpha/\rho,0))$ be the $\rho$-margin loss, $R^\e_*=\inf_f R^\e(f)$, and $R^\e_{\phi_\rho,*}(f)=\inf_f R_{\phi_\rho}^\e(f)$. Then 
               \[R^\e(f)-R^\e_*\leq R^\e_{\phi_\rho}(f)-R^\e_{\phi_\rho,*}.\]
        \end{theorem}
    Notice that this theorem immediately implies that the $\rho$-margin loss is in fact adversarially consistent. The proof below is completely independent of the argument in Section~\ref{sec:adversarially_consistent_losses}.
    \begin{proof}
        Notice that for the $\rho$-margin loss, $C_{\phi_\rho}^*=C^*$ and therefore, 
        the optimal $\phi_\rho$-risk $R^\e_{{\phi_\rho},*}$ 
        equals the optimal adversarial classification risk $R^\e_*$. However, since $\phi_{\rho}(\alpha)\geq \one_{\alpha \leq 0}$ and $\phi_\rho(-\alpha)\geq \one_{\alpha>0}$ for any $\alpha$, one can conclude that $R^\e(f)\leq R^\e_{\phi_\rho}(f)$. Therefore,
        \[R^\e(f)-R^\e_*= R^\e(f)-R^\e_{{\phi_\rho},*}\leq R^\e_{\phi_\rho}(f)-R^\e_{\phi_\rho,*}\]
    \end{proof}
    
    This bound implies that reducing the excess adversarial $\rho$-margin loss by $\Delta$ also reduces an upper bound on the excess adversarial classification loss by $\Delta$. Thus, one would expect that minimizing the adversarial $\rho$-margin risk would be an effective procedure for minimizing the adversarial classification risk.

    Extending Theorem~\ref{th:tight_bound} to other losses remains an open problem. In the non-adversarial scenario, many prior works develop techniques for computing such bounds. These include the $\Psi$-transform of \citep{BartlettJordanMcAuliffe2006}, calibration analysis in \citep{Steinwart2007}, and special techniques for proper losses in \citep{ReidWilliamson2009}. 
    
    Contemporary work \citep{MaoMohriZhong2023crossentropy} derives an $\cH$-consistency surrogate risk bound for a variant of the adversarial $\rho$-margin loss.

%% file: Sections/6-conclusion.tex
    \section{Conclusion}
        In conclusion, we proved that the adversarial training procedure is consistent for perturbations in an $\e$-ball if an only if $C_\phi^*(1/2)<\phi(0)$. The technique that proved consistency extends to perturbation sets which satisfy existence and minimax theorems analogous to Theorems~\ref{th:minimax_classification} and~\ref{th:minimax}. Furthermore, we showed a quantitative excess risk bound for the adversarial $\rho$-margin loss. Finding such bounds for other losses remains an open problem. We hope that insights to consistency and the structure of
adversarial learning will lead to the design of better adversarial learning algorithms.

%% file: Sections/7-acknowledgements.tex
\begin{ack}
Natalie Frank was supported in part by the Research Training Group in Modeling and Simulation funded by the National Science Foundation via grant RTG/DMS – 1646339. Jonathan Niles-Weed was supported in part by a Sloan Research Fellowship.
\end{ack}

%% file: Appendices/1-Alt_consistency.tex
\section{An Alternative Characterization of Consistency-- Proof of Proposition~\ref{prop:alt_consistency_characterization}}\label{app:alt_consistency_characterization}

First, prior work computes the minimum standard $\phi$-risk.
\begin{lemma}
\label{lemma:R_phi_min}
    Let $\phi$ be any monotonic loss function. Then 
    \[\inf_{f\text{ measurable}}R_\phi(f)=\int C_\phi^*(\eta)d\PP\]
    %\tdd{Do we also need to include the statement $\inf_f R(f)=\int C^*(\eta)d\PP$?}
\end{lemma}
This result appears on page 4 of \citep{BartlettJordanMcAuliffe2006}. Notice that Lemma~\ref{lemma:R_phi_min} is Theorem~\ref{th:minimax} with $\e=0$.   
Next, one can use the following lemma to compare minimizing sequences of $C_\phi(\eta,\cdot)$ and $C(\eta,\cdot)$.

\begin{lemma}\label{lemma:phi_zero_value}
        Assume that Assumption~\ref{as:phi} holds, $\phi$ is consistent, and $0\in \argmin C_\phi(\eta,\cdot)$. Then $\eta=1/2$. 
    \end{lemma}
    \begin{proof}
        Consider a distribution for which $\eta(\bx)\equiv \eta$ is constant. Then $R_\phi(f)=C_\phi(\eta,f)$ and $R(f)=C(\eta,f)$. The consistency of $\phi$ implies that if $0$ minimizes $C_\phi(\eta,\cdot)$, then it also must minimize $C(\eta,\cdot)$ and therefore $\eta\leq 1/2$. 
        
        However, notice that $C_\phi(\eta,\alpha)=C_\phi(1-\eta,-\alpha)$. Thus if 0 minimizes $C_\phi(\eta,\cdot)$ it must also minimize $C_\phi(1-\eta,\cdot)$.
        The consistency of $\phi$ then implies that $1-\eta\leq 1/2$ as well and consequently, $\eta=1/2$.
    \end{proof}

    We use this result to prove Proposition~\ref{prop:alt_consistency_characterization} together with a standard argument from analysis:
        \begin{lemma}\label{lemma:subsequence_equivalence}
	Let $\{a_n\}$ be a sequence in $\Rset\cup \{\infty\}$. Then the following are equivalent:
	\begin{enumerate}[label=\arabic*)]
		\item \label{it:subsequence_equivalence_first}$\lim_{n\to \infty} a_n=a$
		\item \label{it:subsequence_equivalence_second}Every subsequence $\{a_{n_j}\}$ of $\{a_n\}$ has a subsequence $\{a_{j_k}\}$ for which $\lim_{k\to\infty} a_{j_k}=a$
	\end{enumerate}
\end{lemma}
%\begin{proof}[Proof of Lemma~\ref{lemma:subsequence_equivalence}]To show that \ref{it:subsequence_equivalence_second} implies \ref{it:subsequence_equivalence_first}, we proceed by contrapositive. Assume that there is a subsequence $\{a_{n_j}\}$ of $\{a_n\}$ which has a further subsequence $\{a_{n_{j_k}}\}$ for which $a_{n_{j_k}}\not \to a$. Then $\{a_n\}$ has a subsequence that does not converge to $a$, so $\lim_{n\to \infty} a_n\neq a$.
%\end{proof}
    As a result:
    \begin{corollary}\label{cor:subseq_const}
        If every minimizing sequence $f_n$ of $R_\phi$ has a subsequence $f_{n_j}$ that minimizes $R$, then $\phi$ is consistent.
    \end{corollary}
Furthermore, this corollary can be applied to a distribution with constant $\eta(\bx)$ to conclude:
\begin{corollary}\label{cor:subseq_const_C}
    If every minimizing sequence $\alpha_n$ for $C_\phi(\eta,\cdot)$ has a subsequence $\alpha_{n_j}$ that minimizes $C(\eta,\cdot)$ then one can conclude that every minimizing sequence is of $C_\phi(\eta,\cdot)$ is also a minimizing sequence of $C(\eta,\cdot)$.
\end{corollary}
    
   We now prove a result slightly stronger than Proposition~\ref{prop:alt_consistency_characterization}.
   \begin{theorem}
        The following are equivalent:
        \begin{enumerate}[label=\arabic*)]
            \item\label{it:R_sequence}
                For all distributions, $f_n$ is a minimizing sequence of $R_\phi$ implies that $f_n$ is a minimizing sequence of $R$.
            \item\label{it:C_sequence} For all $\eta\in [0,1]$, $\alpha_n$ is a minimizing sequence of $C_\phi(\eta,\cdot)$ implies that $\alpha_n$ is a minimizing sequence of $C(\eta,\cdot)$.
            \item\label{it:C_minimizers} Every minimizer of $C_\phi(\eta,\cdot)$ is also a minimizer of $C(\eta,\cdot)$.
            \item \label{it:R_minimizers} Every minimizer of $R_\phi$ is a minimizer of $R$
        \end{enumerate}
   \end{theorem}
   The proof is essentially the ``pointwise" argument discussed in Section~\ref{sec:problem_setup}.
   \begin{proof}
       We show that \ref{it:R_sequence} $\Leftrightarrow$ \ref{it:C_sequence}, \ref{it:C_sequence} $\Leftrightarrow$ \ref{it:C_minimizers}, and \ref{it:C_minimizers} $\Leftrightarrow$ \ref{it:R_minimizers}.\\\\
       \textbf{Showing \ref{it:R_sequence} is equivalent to \ref{it:C_sequence}:} \\\\
       To show that \ref{it:R_sequence} implies \ref{it:C_sequence}, consider a distribution for which $\eta(\bx)\equiv \eta$ is constant. 

       For the other direction,  let $f_n$ be any minimizing sequence of $R_\phi$. Then $C_\phi(\eta,f_{n})\geq C_\phi^*(\eta)$ and Lemma~\ref{lemma:R_phi_min} implies that the sequence $C_\phi(\eta,f_{n})$ actually converges to $C_\phi^*(\eta)$ in $L^1(\PP)$. Thus one can pick a subsequence $f_{n_j}$ for which $C_\phi(\eta,f_{n_{j}})$ converges to $C_\phi^*(\eta)$ $\PP$-a.e. (See for instance Corollary~2.32 of \citep{folland}). Then \ref{it:C_sequence} implies that the function sequence $f_{n_{j}}$ minimizes $C(\eta,\cdot)$ and therefore it also minimizes $R$ by Corollary~\ref{cor:subseq_const}. 
       \\\\
       \textbf{Showing \ref{it:C_sequence} is equivalent to \ref{it:C_minimizers}:} \\\\
       To show that \ref{it:C_sequence} implies \ref{it:C_minimizers}, notice that if $\alpha$ is a minimizer of $C_\phi(\eta,\cdot)$, \ref{it:C_sequence} immediately implies that the sequence $\alpha_n\equiv \alpha$ also minimizes $C(\eta,\cdot)$.

       For the other direction, assume that every minimizer of $C_\phi(\eta,\cdot)$ is also a minimizer of $C(\eta,\cdot)$. Let $\alpha_n$ be a minimizing sequence of $C_\phi(\eta,\cdot)$. Over the extended real numbers $\ov \Rset$, $\alpha_n$ has a subsequence  $\alpha_{n_{j}}$ that converges to a limit point $a$, which must be a minimizer of $C_\phi(\eta,\cdot)$.  Now if $a\neq 0$, both $\one_{\alpha\leq 0}, \one_{\alpha>0}$ are continuous at $a$ so that one can conclude that $\alpha_{n_{j}}$ also minimizes $C(\eta,\cdot)$. If in fact $a=0$, Lemma~\ref{lemma:phi_zero_value} implies that $\eta=1/2$ and thus \emph{any} $\alpha$ minimizes $C(1/2,\cdot)$. Thus Corollary~\ref{cor:subseq_const_C} implies that $\alpha_n$ minimizes $C(\eta,\cdot)$.
       \\\\
       \textbf{Showing \ref{it:C_minimizers} is equivalent to \ref{it:R_minimizers}}\\\\
       To show that \ref{it:R_minimizers} implies \ref{it:C_minimizers}, consider a distribution for which $\eta(\bx)\equiv\eta$ is constant. 

       For the other direction, let $f^*$ be a minimizer of $R_\phi$. Then $C_\phi(\eta(\bx),f^*(\bx))\geq C_\phi^*(\eta(\bx))$ but $R_\phi(f^*)=\int C_\phi^*(\eta)d\PP$ by Lemma~\ref{lemma:R_phi_min}. Therefore $C_\phi(\eta(\bx),f^*(\bx))=C_\phi^*(\eta(\bx))$ $\PP$-a.e. Item~\ref{it:C_minimizers} then implies the result.
   \end{proof}

%% file: Appendices/2-ov_R_to_R.tex
\section{Minimizing $R_\phi^\e$ over $\ov \Rset$-valued functions}\label{app:Rset-approximation}
In this appendix, we will show 
\begin{lemma}\label{lemma:ov_Rset_to_Rset}
    Let $R_\phi^\e$ be defined as in \eqref{eq:adv_phi_loss}. Then
    \[ \inf_{\substack{f \text{ Borel,}\\ f\text{ $\Rset$-valued}}} \prm(f)= \inf_{\substack{f \text{ Borel,}\\ f\text{ $\ov \Rset$-valued}}} \prm(f)\]
\end{lemma}
Integrals of functions assuming values in $\Rset\cup\{\infty\}$ can still be defined using standard measure theory, see for instance \citep{folland}. 

Recall that \citep{FrankNilesWeed23minimax} originally proved their minimax result for $\ov\Rset$-valued functions and thus this lemma is essential for the statement of Theorem~\ref{th:minimax}.

%We will apply the following lemma proved in \citep{FrankNilesWeed23minimax}. We reproduce the proof here for completeness.
%\begin{lemma}\label{lemma:liminf_sup_swap}
%    Let $h_n$ be a sequence of functions. Then $\lim_{n\to \infty}S_\e(h_n)\geq S_\e(\liminf h_n)$
%\end{lemma}
%\begin{proof}
%        \begin{align*}
%    &\liminf_{n\to \infty} S_\e(h_n)(\bx)=\liminf_{n\to \infty} \sup_{\|\bh\|\leq \e} h_n(\bx+\bh)=\sup_N \inf_{n\geq N} \sup_{\|\bh\|\leq \e}h_n(\bx+\bh)\\
%    &\geq  \sup_{\|\bh\|\leq \e}\sup_N \inf_{n\geq N}h_n(\bx+\bh) = \sup_{\|\bh\|\leq \e} \liminf_{n\to \infty} h_n(\bx+\bh)= S_\e( \liminf_{n\to \infty} h_n)(\bx)
%\end{align*}
%\end{proof}

\begin{proof}[Proof of Lemma~\ref{lemma:ov_Rset_to_Rset}]
    Let $f$ be an $\ov \Rset$-valued function for with $R_\phi^\e(f)<\infty$. We will show that the truncation $f_N=\min(\max(f,-N),N)$ satisfies $\lim_{N\to \infty} R^\e_\phi(f_N)= R^\e_\phi(f)$. Lemma~\ref{lemma:ov_Rset_to_Rset} then follows from this statement. 

    Define a function $\sigma_{[a,b]}\colon \ov \Rset\to [a,b]$ by 
    \[\sigma_{[a,b]}(\alpha)=\begin{cases}
        b &\alpha> b\\
        \alpha &\alpha \in [a,b]\\
        a &\alpha <a
    \end{cases}\]
    Notice that $\sigma_{[a,b]}(-\alpha)=-\sigma_{[-b,-a]}(\alpha)$. Thus if $a=-b$, then $\sigma_{[a,b]}$ is anti-symmetric. Furthermore, because $\phi$ is continuous and non-increasing, for any function $g$,
    \[\phi(\sigma_{[a,b]}(g))=\sigma_{[\phi(b),\phi(a)]}( \phi(g))\]
    and as $\sigma_{[a,b]}(\alpha)$ is continuous and non-decreasing,
    \[S_\e(\sigma_{[a,b]}(g))=\sigma_{[a,b]}(S_\e(g))\]
    Now let $f_N=\sigma_{[-N,N]}(f)$. Then $S_\e(\phi \circ f_N)$, $S_\e(\phi \circ -f_N)$ satisfy
    \[S_\e(\phi(f_N))=\sigma_{[\phi(N),\phi(-N)]}(S_\e(\phi \circ f)),\quad S_\e(\phi(-f_N))=\sigma_{[\phi(N),\phi(-N)]}(S_\e(\phi \circ -f))\]
    
    Therefore, $S_\e(\phi \circ f_N)$ ,$S_\e(\phi\circ -f_N)$ converge pointwise to $S_\e(\phi \circ f)$, $S_\e(\phi\circ -f)$. Furthermore, for $N\geq 1$, $ \phi(f_N)\leq \phi(f)+\phi(1)$ which is integrable with respect to $\PP_1$. Similarly, $ \phi(-f_N)\leq \phi(-f)+\phi(1)$ which is integrable with respect to $\PP_0$. Therefore, the dominated convergence theorem implies that 
    \begin{align*}
        &\lim_{N\to \infty} R_\phi^\e(f_N)=\prm(f)
    \end{align*}
\end{proof}

%% file: Appendices/3-phi_properties.tex
\section{Further Properties of Adversarially Consistent Losses-- Proofs of Lemma~\ref{lemma:a_def_main}, Lemma~\ref{lemma:minimizing_seq}, and Proposition~\ref{prop:a_def_consistent} }\label{app:further_properties_of_losses}

Recall the condition $C_\phi^*(1/2)<\phi(0)$ implies that minimizers of $C_\phi(1/2,\alpha)$ are bounded away from zero. Lemma~\ref{lemma:a_def_definition} states that this property actually holds for \emph{all} $\eta$. To prove this fact, we decompose $C_\phi(\eta,\alpha)$ into $C_\phi(1/2,\alpha)$ and a monotonic function: 
    \begin{equation}\label{eq:C_phi_rewrite}
        C_\phi(\eta, \alpha)=\eta \phi(\alpha)+(1-\eta)\phi(-\alpha)=(\eta-1/2)(\phi(\alpha)-\phi(-\alpha))+\frac 12(\phi(\alpha)+\phi(-\alpha)).    
    \end{equation}
    \begin{lemma}\label{lemma:a_def_definition}
        Assume that $C_\phi^*(1/2)<\phi(0)$. Then there exists an $a>0$ for which $|\alpha|<a$ implies $C_\phi(\eta,\alpha)\neq C_\phi^*(\eta)$ for all $\eta$. This $a$ satisfies $\phi(a)<\phi(0)$.
    \end{lemma}
\begin{proof}
    Let $S$ be the set of non-negative minimizers of $C_\phi(1/2,\cdot)$ and define $a=\inf S$. Because $\phi$ is continuous, $a$ is also a minimizer of $C_\phi(1/2,\cdot)$ and thus $C_\phi(1/2,a)=C_\phi^*(1/2)<\phi(0)=C_\phi(1/2,0)$. Therefore, $\phi(a)<\phi(0)$ follows from the fact that $\phi(-a)\geq \phi(0)$. 
    
    We will now show that $C_\phi(\eta,\cdot)$ does not achieve its optimum on $(-a,a)$ for any $\eta$. First, this statement holds for $\eta=1/2$ due to the definition of $a$. Next, we will assume that $\eta>1/2$, the case $\eta<1/2$ is analogous. 
    To start, we can decompose the quantity $C_\phi(\eta,\alpha)$ as in \eqref{eq:C_phi_rewrite}.
    Subsequently, because $a$ is the smallest positive 
    minimizer of $C_\phi(1/2,\cdot)$, $1/2(\phi(\alpha)+\phi(-\alpha))$ assumes its infimum over $[-a,a]$ only at $-a$ and $a$. %Furthermore, on $(-a,a)$, $1/2(\phi(\alpha)+\phi(-\alpha)>1/2(\phi(a)+\phi(-a))$.
    Next, notice that $\phi(\alpha)-\phi(-\alpha)$ is non-increasing on $[-a,a]$. Furthermore, because $\phi(a)<\phi(0)$, one can conclude that $\phi(-a)-\phi(a)>0>\phi(a)-\phi(-a)$, and thus
    the function $\alpha \mapsto \phi(\alpha)-\phi(-\alpha)$ is non-constant on $[-a,a]$.
    Therefore, \eqref{eq:C_phi_rewrite} achieves its optimum over $[-a,a]$ only at $\alpha=a$. 
    Thus, any $\alpha \in (-a,a)$ cannot be a minimizer of $C_\phi(\eta,\cdot)$ because $C_\phi(\eta,\alpha)>C_\phi(\eta,a)\geq C_\phi^*(\eta)$.

\end{proof}

\begin{proof}[Proof of Lemma~\ref{lemma:a_def_main}]
    Lemma~\ref{lemma:a_def_definition} (above) immediately implies the forward direction.

    For the backwards direction, note that if there is an $a$ for which $|\alpha^*|\geq a$ for any minimizer $C_\phi(\eta,\cdot)$ for all $\eta$, then $0$ does not minimize $C_\phi(1/2,\cdot)$. Therefore $C_\phi^*(1/2)<C_\phi(1/2,0)=\phi(0)$.

\end{proof}

\begin{proof}[Proof of Proposition~\ref{prop:a_def_consistent}]

    We will argue that for each $\eta$, every minimizer of $C_\phi(\eta,\cdot)$ over $\ov \Rset$ is also a minimizer of $C(\eta,\cdot)$. Proposition~\ref{prop:alt_consistency_characterization} will then imply that $\phi$ is consistent. To start, notice that \emph{every} $\alpha$ is a minimizer of $C(1/2,\cdot)$. Next, we will show that for $\eta>1/2$, every minimizer of $C_\phi(\eta,\cdot)$ is also a minimizer of $C(\eta,\cdot)$. The argument for $\eta<1/2$ is analogous.

    Consider the decomposition of $C_\phi(\eta,\alpha)$ in \eqref{eq:C_phi_rewrite}. Let $a$ be as in Lemma~\ref{lemma:a_def_definition} and notice that if $\alpha>a$ then $\phi(\alpha)<\phi(-\alpha)$. Hence as $\eta>1/2$, then $C_\phi(\eta,\alpha)<C_\phi(\eta,-\alpha)$. Furthermore, Lemma~\ref{lemma:a_def_definition} implies that there is no minimizer to $C_\phi(\eta,\cdot)$ in $(-a,a)$ and thus every minimizer to $C_\phi(\eta,\cdot)$ must be strictly positive. Therefore, every minimizer of $C_\phi(\eta,\cdot)$ also minimizes $C(\eta,\cdot)$.
    
   % Let $c$ be as in Lemma~\ref{lemma:minimizing_seq}. Then Lemma~\ref{lemma:minimizing_seq} implies that $C_\phi(\eta,\cdot)$ does not assume its minimum on $(-c,c)$. If $\alpha\leq -c$, then $\phi(\alpha)\geq \phi(-c)\geq \phi(0)> \phi(c)\geq \phi(-\alpha)$ due to \ref{as:phi} and hence $C_\phi(\eta,-\alpha)<C_\phi(\eta,\alpha)$. Therefore, the minimizer of $C_\phi(\eta,\cdot)$ must be in $[c,+\infty]$. However, when $\eta>1/2$, every $\alpha\in [c,+\infty)$ minimizes $C(\eta,\cdot)$. Thus if $\eta>1/2$, every minimizer of $C_\phi(\eta,\cdot)$ is also a minimizer of $C(\eta,\cdot)$.
\end{proof}

 Next, Lemma~\ref{lemma:minimizing_seq} is a quantitative version of Lemma~\ref{lemma:a_def_definition}.

\begin{proof}[Proof of Lemma~\ref{lemma:minimizing_seq}]
    Let $a$ be as in Lemma~\ref{lemma:a_def_definition} and define $\iphi$ by 
    \begin{equation*}%\label{eq:iphi_definition}
        \iphi(y)=\sup\{\alpha: \phi(\alpha)\geq y\}.
    \end{equation*}
    The function $\iphi$ is the right inverse of $\phi$--- this function satisfies $\phi(\iphi(y))=y$ while $\iphi(\phi(\alpha))\geq \alpha$.

    Set $k =1/2(\phi(0)+\phi(a))$, $c=\iphi(k)=\sup \{ \alpha\colon \phi(\alpha)\geq k\}$. From the definition of $c$, one can conclude that $\alpha>c$ implies that $\phi(\alpha)<\phi(c)$.

    Because $\phi(a)<k=\phi(c)<\phi(0)$ and $\phi$ is non-increasing, $0<c<a$. Thus $[-c,c]\subset (-a,a)$ and Lemma~\ref{lemma:a_def_definition} implies that for all $\alpha\in [-c,c]$ and $\eta\in [0,1]$, $C_\phi(\eta,\alpha)-C_\phi^*(\eta)>0$. As this expression is jointly continuous in the variables $\eta$, $\alpha$ and $[-c,c]\times [0,1]$ is compact, one can define
    \[\delta=\inf_{\substack{\alpha\in [-c,c]\\ \eta\in [0,1]}} C_\phi(\eta,\alpha)-C_\phi^*(\eta) \]
    and then it holds that $\delta>0$ and $C_\phi(\eta,\alpha)\geq C_\phi^*(\eta)+\delta$ for all $\alpha\in [-c,c]$.

\end{proof}

%% file: Appendices/4-W_infty_couplings.tex
\section{Optimal Transport Facts--- Proof of Lemma~\ref{lemma:S_e_inequality}}\label{app:W_infty}

    \begin{proof}[Proof of Lemma~\ref{lemma:S_e_inequality}]
            Let $\QQ'$ be any measure with $W_\infty(\QQ',\QQ)\leq  \e$. Let $\gamma$ be a coupling with marginals $\QQ$ and $\QQ'$ for which $\esssup_{(\bx,\by)\sim \gamma} \|\bx-\by\|\leq \e$. This measure $\gamma$ is supported on $\Delta_\e=\{(\bx,\by)\colon \|\bx-\by\|\leq \e\}$. Then
    \begin{align*}
        &\int gd\QQ'=\int g(\bx') d\gamma(\bx,\bx')=\int g(\bx')\one_{\|\bx'-\bx\|\leq \e} d\gamma(\bx,\bx')\\
        &\leq \int S_\e(g)(\bx)\one_{\|\bx'-\bx\|\leq \e} d\gamma(\bx,\bx')=\int S_\e(g)(\bx) d\gamma(\bx,\bx')=\int S_\e(g)d\QQ    
    \end{align*}
    %Therefore, we can conclude that 
   % \[\sup_{\QQ'\in \Wball \e(\QQ)} \int g d\QQ'\leq \int S_\e(g)d\QQ.\]

    \end{proof}

%% file: Appendices/5-minimax_classification_proof.tex
\section{Proof of Theorem~\ref{th:minimax_classification}}\label{app:minimax_classification_proof}
As observed in Section~\ref{sec:quantitative}, the $\rho$-margin loss satisfies $R_{\phi_\rho}^\e(f)\geq R^\e(f)$ while $C_{\phi_\rho}^*(\eta)=C^*(\eta)$. Theorem~\ref{th:minimax} then implies that 
\[\sup_{\substack{\PP_0'\in \Wball\e (\PP_0)\\ \PP_1'\in \Wball\e(\PP_1)}} \bar R(\PP_0',\PP_1')=\sup_{\substack{\PP_0'\in \Wball\e (\PP_0)\\ \PP_1'\in \Wball\e(\PP_1)}} \bar R_{\phi_\rho}(\PP_0',\PP_1')=\inf_f R_{\phi_\rho}^\e(f)\geq \inf_f R^\e(f)\]

The opposite inequality follows from swapping an $\inf$ and a $\sup$--- a form of weak duality. We prove this weak duality for $\ov \Rset=\Rset \cup\{-\infty,+\infty\}$-valued functions in order to later apply a result from \citep{FrankNilesWeed23minimax} which is also stated for $\ov \Rset$-valued functions.

\begin{lemma}[Weak Duality]\label{lemma:weak_duality}
    Let $R^\e$ be the adversarial classification loss.
      Then 
     \begin{equation}\label{eq:weak_duality}
         \inf_{\substack{f\text{ Borel,}\\ f\text{ $\ov \Rset$-valued}}} R^\e(f) \geq \sup_{\substack{\PP_0'\in \Wball\e (\PP_0)\\ \PP_1'\in \Wball\e(\PP_1)}} \bar R(\PP_0',\PP_1')
     \end{equation}
\end{lemma}
\begin{proof}
    Notice that Lemma~\ref{lemma:S_e_inequality} implies that for any function $g$,
    \[\int S_\e(g)d\QQ\geq \sup_{\QQ'\in \Wball \e(\QQ)} \int gd\QQ'.\]
    Applying this inequality to the functions $\one_{f\leq 0},\one_{f>0}$ in the expression for $R^\e(f)$ results in
    \[\int S_\e(\one_{f\leq 0})d\PP_1+\int S_\e(\one_{f>0} )d\PP_0\geq \sup_{\substack{\PP_0'\in \Wball\e (\PP_0)\\ \PP_1'\in \Wball\e(\PP_1)}} \int \one_{f\geq 0} d\PP_1'+\int \one_{f<0}d\PP_0'\]
    Thus by swapping the $\inf$ and the $\sup$ and defining $\PP'=\PP_0'+\PP_1'$, $\eta'=d\PP_1'/d\PP'$,
    \begin{align*}
        &\inf_{\substack{f\text{ Borel}\\f\text{ $\ov \Rset$-valued}}} \int S_\e(\one_{f\leq 0})d\PP_1+\int S_\e(\one_{f>0})d\PP_0\geq \inf_{\substack{f\text{ Borel}\\ f\text{ $\ov \Rset$-valued}}}\sup_{\substack{\PP_0'\in \Wball\e (\PP_0)\\ \PP_1'\in \Wball\e(\PP_1)}} \int \one_{f\leq 0}d\PP_1'+\int \one_{f>0}d\PP_0'\\
        &\geq \sup_{\substack{\PP_0'\in \Wball\e (\PP_0)\\ \PP_1'\in \Wball\e(\PP_1)}} \inf_{\substack{f \text{ Borel}\\f\text{ $\ov \Rset$-valued}}}\int \one_{f\leq 0}d\PP_1'+\int \one_{f>0}d\PP_0'\\
        %&=\sup_{\substack{\PP_0'\in \Wball\e (\PP_0)\\ \PP_1'\in \Wball\e(\PP_1)}}\inf_{f\text{ Borel}}\int \frac{d\PP_1'}{d(\PP_0'+\PP_1')} \one_{f\leq 0}+\left( 1-\frac{d\PP_1'}{d(\PP_0'+\PP_1')} \right) \one_{f>0} d(\PP_0'+\PP_1')\\
        &=\sup_{\substack{\PP_0'\in \Wball\e (\PP_0)\\ \PP_1'\in \Wball\e(\PP_1)}}\inf_{\substack{f\text{ Borel}\\f\text{ $\ov \Rset$-valued}}}\int C(\eta',f) d\PP'\geq \sup_{\substack{\PP_0'\in \Wball\e (\PP_0)\\ \PP_1'\in \Wball\e(\PP_1)}}\int C^*(\eta') d\PP'
        =\sup_{\substack{\PP_0'\in \Wball\e (\PP_0)\\ \PP_1'\in \Wball\e(\PP_1)}}\bar R(\PP_0',\PP_1')
    \end{align*}
\end{proof}

Strong duality and existence of maximizers/minimizers then follows from weak duality.
\begin{proof}[Proof of Theorem~\ref{th:minimax_classification}]
    Let $\phi_\rho(\alpha)$ be the $\phi$-margin loss $\phi_\rho=\min(1,\max( 1-\alpha/\rho,0))$. Then as discussed in Section~\ref{sec:quantitative}, one can bound the adversarial classification risk $R^\e(f)$ by $R^\e(f)\leq R_{\phi_\rho}^\e(f)$ but $C_{\phi_\rho}^*(\eta)=C^*(\eta)$ and thus $\bar R_{\phi_\rho}=\bar R$. 
    
    The minimax theorem for surrogate losses in \citep{FrankNilesWeed23minimax} (Theorem~6) states that there is an $\ov \Rset$-valued function $f^*$, and measures $\PP_0^*,\PP_1^*$ for which $R_{\phi_\rho}^\e(f^*)=\bar R_{\phi_\rho}(\PP_0^*,\PP_1^*)$. Thus weak duality (Lemma~\ref{lemma:weak_duality}) implies \[\bar R_{\phi_\rho}(\PP_0^*,\PP_1^*)=\bar R(\PP_0^*,\PP_1^*)\leq R^\e(f^*)\leq R_{\phi_\rho}^\e(f^*).\] However, the fact that $R_{\phi_\rho}^\e(f^*)=\bar R_{\phi_\rho}(\PP_0^*,\PP_1^*)$ implies that the inequalities above must actually be equalities. This relation proves strong duality for the adversarial classification risk (Equation~\ref{eq:minimax_classification}) and that $f^*$ minimizes $R^\e$ and $(\PP_0^*,\PP_1^*)$ maximizes $\bar R$ over $\Wball \e(\PP_0)\times \Wball \e(\PP_1)$. 

    Next, let $\hat f=\min(1,\max(\hat f,-1))$. Then $\hat f$ is $\Rset$-valued and $R^\e(\hat f)=R^\e(f^*)$. Thus $\hat f$ is an $\Rset$-valued minimizer of $R^\e$.    
    \end{proof}

%% file: Appendices/6-ineq_to_eq.tex
\section{Proof of Lemma~\ref{lemma:classification_comp_slack_easy}}\label{app:easy}
\begin{proof}[Proof of Lemma~\ref{lemma:classification_comp_slack_easy}]
    Lemma~\ref{lemma:S_e_inequality} implies that for each $n$,
    \begin{equation*}
    	\int S_\e(\one_{f_n\leq 0})d\PP_1\geq \int \one_{f_n\leq 0} d\PP_1^*\,.
    \end{equation*}
    Therefore, writing $\ell_n =\int S_\e(\one_{f_n\leq 0})d\PP_1$ and $r_n = \int \one_{f_n\leq 0} d\PP_1^*$, we have
    \begin{equation}
    	\liminf_{n\to \infty} r_n \leq  \liminf_{n\to \infty} \ell_n \leq \limsup_{n\to \infty} \ell_n\,.
    \end{equation}
Therefore, \eqref{eq:sup_comp_slack_approx_classification_mod_1} implies both that that the limit $\lim_{n\to \infty} \int S_\e(\one_{f_n\leq 0})d\PP_1$ exists and that 
    \begin{equation}\label{eq:ls_li_match}
        \lim_{n\to \infty} \int S_\e(\one_{f_n\leq 0})d\PP_1= \liminf_{n\to \infty}\int  \one_{f_n\leq 0}d\PP_1^*
    \end{equation}
    
    Similarly, because $\limsup_{n\to \infty} \ell_n\geq \limsup_{n\to \infty} r_n\geq \liminf_{n\to \infty} r_n$, the relation \eqref{eq:sup_comp_slack_approx_classification_mod_1} implies that the limit $\lim_{n\to \infty} \int \one_{f_n\leq 0} d\PP_1^*$ exists. The first relation of \eqref{eq:sup_comp_slack_approx_classification} then follows from \eqref{eq:ls_li_match} and the existence of the limit of %$\int S_\e(\one_{f_n\leq 0})d\PP_1$, 
    $\int \one_{f_n\leq 0} d\PP_1^*$.

    An analogous argument shows that \eqref{eq:sup_comp_slack_approx_classification_mod_0} implies the second relation of \eqref{eq:sup_comp_slack_approx_classification}.
\end{proof}

%% file: main.bbl
\begin{thebibliography}{28}
\providecommand{\natexlab}[1]{#1}
\providecommand{\url}[1]{\texttt{#1}}
\expandafter\ifx\csname urlstyle\endcsname\relax
  \providecommand{\doi}[1]{doi: #1}\else
  \providecommand{\doi}{doi: \begingroup \urlstyle{rm}\Url}\fi

\bibitem[Awasthi et~al.(2021{\natexlab{a}})Awasthi, Frank, Mao, Mohri, and
  Zhong]{AwasthiFrankMao2021}
P.~Awasthi, N.~Frank, A.~Mao, M.~Mohri, and Y.~Zhong.
\newblock Calibration and consistency of adversarial surrogate losses.
\newblock \emph{NeurIps}, 2021{\natexlab{a}}.

\bibitem[Awasthi et~al.(2021{\natexlab{b}})Awasthi, Frank, and
  Mohri]{AwasthiFrankMohri2021}
P.~Awasthi, N.~S. Frank, and M.~Mohri.
\newblock On the existence of the adversarial bayes classifier (extended
  version).
\newblock \emph{arxiv}, 2021{\natexlab{b}}.

\bibitem[Awasthi et~al.(2021{\natexlab{c}})Awasthi, Mao, Mohri, and
  Zhong]{AwasthiMaoMohri}
P.~Awasthi, A.~Mao, M.~Mohri, and Y.~Zhong.
\newblock A finer calibration analysis for adversarial robustness.
\newblock \emph{arxiv}, 2021{\natexlab{c}}.

\bibitem[Awasthi et~al.(2022)Awasthi, Mao, Mohri, and
  Zhong]{AwasthiMaoMohriZhong22}
P.~Awasthi, A.~Mao, M.~Mohri, and Y.~Zhong.
\newblock H-consistency bounds for surrogate loss minimizers.
\newblock In \emph{Proceedings of the 39th International Conference on Machine
  Learning}. PMLR, 2022.

\bibitem[Bao et~al.(2021)Bao, Scott, and Sugiyama]{bao2021calibrated}
H.~Bao, C.~Scott, and M.~Sugiyama.
\newblock Calibrated surrogate losses for adversarially robust classification.
\newblock \emph{arxiv}, 2021.

\bibitem[Bartlett et~al.(2006)Bartlett, Jordan, and
  McAuliffe]{BartlettJordanMcAuliffe2006}
P.~L. Bartlett, M.~I. Jordan, and J.~D. McAuliffe.
\newblock Convexity, classification, and risk bounds.
\newblock \emph{Journal of the American Statistical Association}, 101\penalty0
  (473), 2006.

\bibitem[Ben-David et~al.(2003)Ben-David, Eiron, and
  Long]{BenDavidEironLong2003}
S.~Ben-David, N.~Eiron, and P.~M. Long.
\newblock On the difficulty of approximately maximizing agreements.
\newblock \emph{Journal of Computer System Sciences}, 2003.

\bibitem[Bhattacharjee and Chaudhuri(2020)]{bhattacharjee2020nonparametric}
R.~Bhattacharjee and K.~Chaudhuri.
\newblock When are non-parametric methods robust?
\newblock \emph{PMLR}, 2020.

\bibitem[Bhattacharjee and Chaudhuri(2021)]{bhattacharjee2021consistent}
R.~Bhattacharjee and K.~Chaudhuri.
\newblock Consistent non-parametric methods for maximizing robustness.
\newblock \emph{NeurIps}, 2021.

\bibitem[Biggio et~al.(2013)Biggio, Corona, Maiorca, Nelson, {\v{S}}rndi{\'c},
  Laskov, Giacinto, and Roli]{biggio2013evasion}
B.~Biggio, I.~Corona, D.~Maiorca, B.~Nelson, N.~{\v{S}}rndi{\'c}, P.~Laskov,
  G.~Giacinto, and F.~Roli.
\newblock Evasion attacks against machine learning at test time.
\newblock In \emph{Joint European conference on machine learning and knowledge
  discovery in databases}, pages 387--402. Springer, 2013.

\bibitem[Bungert et~al.(2021)Bungert, Trillos, and
  Murray]{BungertGarciaMurray2021}
L.~Bungert, N.~G. Trillos, and R.~Murray.
\newblock The geometry of adversarial training in binary classification.
\newblock \emph{arxiv}, 2021.

\bibitem[Folland(1999)]{folland}
G.~B. Folland.
\newblock \emph{Real analysis: modern techniques and their applications},
  volume~40.
\newblock John Wiley \& Sons, 1999.

\bibitem[Frank and Niles-Weed(2023)]{FrankNilesWeed23minimax}
N.~S. Frank and J.~Niles-Weed.
\newblock Existence and minimax theorems for adversarial surrogate risks in
  binary classification.
\newblock \emph{arXiv}, 2023.

\bibitem[Li et~al.(2021)Li, Xu, Xiao, Li, and Shen]{LiXuTraffic}
Y.~Li, X.~Xu, J.~Xiao, S.~Li, and H.~T. Shen.
\newblock Adaptive square attack: Fooling autonomous cars with adversarial
  traffic signs.
\newblock \emph{IEEE Internet of Things Journal}, 8\penalty0 (8), 2021.

\bibitem[Lin(2004)]{Lin2004}
Y.~Lin.
\newblock A note on margin-based loss functions in classification.
\newblock \emph{Statistics \& Probability Letters}, 68\penalty0 (1):\penalty0
  73--82, 2004.

\bibitem[Mao et~al.(2023)Mao, Mohri, and Zhong]{MaoMohriZhong2023crossentropy}
A.~Mao, M.~Mohri, and Y.~Zhong.
\newblock Cross-entropy loss functions: Theoretical analysis and applications,
  2023.

\bibitem[Meunier et~al.(2022)Meunier, Ettedgui, Pinot, Chevaleyre, and
  Atif]{MeunierEttedguietal22}
L.~Meunier, R.~Ettedgui, R.~Pinot, Y.~Chevaleyre, and J.~Atif.
\newblock Towards consistency in adversarial classification.
\newblock \emph{arXiv}, 2022.

\bibitem[Mingyuan~Zhang(2020)]{ZhangAgarwal}
S.~A. Mingyuan~Zhang.
\newblock Consistency vs. h-consistency: The interplay between surrogate loss
  functions and the scoring function class.
\newblock \emph{NeurIps}, 2020.

\bibitem[Paschali et~al.(2018)Paschali, Conjeti, Navarro, and
  Navab]{paschali2018generalizability}
M.~Paschali, S.~Conjeti, F.~Navarro, and N.~Navab.
\newblock Generalizability vs. robustness: Adversarial examples for medical
  imaging.
\newblock \emph{Springer}, 2018.

\bibitem[Philip M.~Long(2013)]{LongServedioH-consistency}
R.~A.~S. Philip M.~Long.
\newblock Consistency versus realizable h-consistency for multiclass
  classification.
\newblock \emph{ICML}, 2013.

\bibitem[Pydi and Jog(2021)]{PydiJog2021}
M.~S. Pydi and V.~Jog.
\newblock The many faces of adversarial risk.
\newblock \emph{Neural Information Processing Systems}, 2021.

\bibitem[Reid and Williamson(2009)]{ReidWilliamson2009}
M.~D. Reid and R.~C. Williamson.
\newblock Surrogate regret bounds for proper losses.
\newblock In \emph{Proceedings of the 26th Annual International Conference on
  Machine Learning}, New York, NY, USA, 2009. Association for Computing
  Machinery.

\bibitem[Steinwart(2007)]{Steinwart2007}
I.~Steinwart.
\newblock How to compare different loss functions and their risks.
\newblock \emph{Constructive Approximation}, 2007.

\bibitem[Szegedy et~al.(2013)Szegedy, Zaremba, Sutskever, Bruna, Erhan,
  Goodfellow, and Fergus]{szegedy2013intriguing}
C.~Szegedy, W.~Zaremba, I.~Sutskever, J.~Bruna, D.~Erhan, I.~Goodfellow, and
  R.~Fergus.
\newblock Intriguing properties of neural networks.
\newblock \emph{arXiv preprint arXiv:1312.6199}, 2013.

\bibitem[Trillos and Murray(2020)]{trillosMurray2020}
N.~G. Trillos and R.~Murray.
\newblock Adversarial classification: Necessary conditions and geometric flows.
\newblock \emph{arxiv}, 2020.

\bibitem[Trillos et~al.(2022)Trillos, Jacobs, and Kim]{TrillosJacobsKim22}
N.~G. Trillos, M.~Jacobs, and J.~Kim.
\newblock The multimarginal optimal transport formulation of adversarial
  multiclass classification.
\newblock \emph{arXiv}, 2022.

\bibitem[Trillos et~al.(2023)Trillos, Jacobs, and Kim]{trillos2023existence}
N.~G. Trillos, M.~Jacobs, and J.~Kim.
\newblock On the existence of solutions to adversarial training in multiclass
  classification, 2023.

\bibitem[Zhang(2004)]{zhang04}
T.~Zhang.
\newblock Statistical behavior and consistency of classification methods based
  on convex risk minimization.
\newblock \emph{The Annals of Statistics}, 2004.

\end{thebibliography}
